\documentclass[10 pt]{article}
\RequirePackage{style}
\title{\normalsize \bfseries \MakeUppercase{An Approximate Policy Iteration Viewpoint of \\Actor-Critic Algorithms}}
\author{\normalsize \scshape Zaiwei Chen\textsuperscript{1} and Siva Theja Maguluri\textsuperscript{2}\\
{\footnotesize
\textsuperscript{1}Caltech CMS, \href{mailto:zchen458@caltech.edu}{\textit{zchen458@caltech.edu}}}\\
{\footnotesize\textsuperscript{2}Georgia Tech ISyE, \href{mailto:siva.theja@gatech.edu}{\textit{siva.theja@gatech.edu}}}
}

\date{\vspace{-0.4 in}}

\begin{document}
	
\setlength\abovedisplayskip{3pt}
\setlength\belowdisplayskip{3pt}	
\maketitle
\pagestyle{fancy}
\voffset = 0pt
\footskip = 0.25 in
\fancyhf{} 
\fancyhead[R]{\thepage}
\fancyhead[C]{\scshape \small Z. Chen and S.T. Maguluri}

\begin{abstract}
In this work, we consider policy-based methods for solving the reinforcement learning problem, and establish the sample complexity guarantees. A policy-based algorithm typically consists of an actor and a critic. We consider using various policy update rules for the actor, including the celebrated natural policy gradient. In contrast to the gradient ascent approach taken in the literature, we view natural policy gradient as an approximate way of implementing policy iteration, and show that natural policy gradient (without any regularization) enjoys geometric convergence when using increasing stepsizes. As for the critic, we consider using TD-learning with linear function approximation and off-policy sampling. Since it is well-known that in this setting TD-learning can be unstable, we propose a stable generic algorithm (including two specific algorithms: the $\lambda$-averaged $Q$-trace and the two-sided $Q$-trace) that uses multi-step return and generalized importance sampling factors, and provide the finite-sample analysis. Combining the geometric convergence of the actor with the finite-sample analysis of the critic, we establish for the first time an overall $\Tilde{\mathcal{O}}(\epsilon^{-2})$  sample complexity for finding an optimal policy (up to a function approximation error) using policy-based methods under off-policy sampling and linear function approximation. 
\end{abstract}

\section{Introduction}
\label{sec:introduction}

In recent years, reinforcement learning (RL) has demonstrated impressive performance in solving practical problems, such as the game of Go \citep{silver2016mastering}, power systems \citep{zhang2019deep}, city navigation \cite{mirowski2018learning}, and nuclear plasma control \citep{degrave2022magnetic}. An RL problem in its nature is a sequential decision-making problem, and is usually modeled as a Markov decision process (MDP). To solve an MDP, the two most well-known methods are value iteration (VI) and policy iteration (PI). However, compared to MDPs, a main feature of RL is that the parameters of the environmental model, such as the transition probabilities and the reward function, are unknown to the agent. Therefore, due to the lack of knowledge about the model, either VI or PI is not directly implementable. To overcome this challenge, data-driven versions of VI and PI were developed. The resulting algorithms are called $Q$-learning, which is the RL-counterpart of VI, and approximate policy iteration (API), which is the RL-counterpart of PI. The analysis of either $Q$-learning or API mostly relies on the properties of the Bellman operator, such as the contraction property and the monotonicity property \citep{bertsekas1996neuro}.

On the other hand, an MDP at its heart is a continuous optimization problem in the space of policies, and a natural approach to solve a continuous optimization problem is to use gradient ascent/descent and its variants. Based on this viewpoint, the policy gradient algorithm was developed, the counterpart of which in RL is the popular actor-critic. Unlike the analysis of VI or PI, the analysis of policy gradient is largely based on the policy gradient theorem, and machineries from continuous optimization, such as the mirror descent analysis \citep{lan2021policy}. In contrast to the geometric convergence of VI or PI, policy gradient in general does not achieve geometric convergence. A variant of policy gradient, known as natural policy gradient, was recently shown to enjoy geometric convergence after being regularized. See Section \ref{subsec:literature} for a more detailed discussion about related work in this field.

In this work, we provide a different viewpoint on natural policy gradient. Specifically, we show that with properly chosen stepsizes, natural policy gradient can be viewed as an approximate version of PI, which enables us to establish the geometric convergence without requiring any regularization. To use the result in an actor-critic framework, we design a single time-scale TD-learning variant for policy evaluation that successfully overcomes the deadly triad. After combining the analysis of the actor and the critic, we establish an overall $\tilde{\mathcal{O}}(\epsilon^{-2})$ sample complexity for finding an optimal policy up to a function approximation error. The more detailed contributions of this work are summarized in the following.

\textit{Geometric Convergence of Natural Policy Gradient.} We show that when using either large enough constant stepsize or geometrically increasing time-varying stepsizes, natural policy gradient is approximately PI. Therefore, by using the contraction property and the monotonicity property of the Bellman operator, we establish geometric convergence of natural policy gradient. Importantly, unlike in existing literature, our result does not require using any kind of regularization techniques in the algorithm design.

\textit{Stable Single Time-Scale Off-Policy TD-Learning under Linear Function Approximation.}
To use our result for natural policy gradient in an actor-critic framework, we also need to specify the algorithm used for policy evaluation (i.e., the critic) and establish the finite-sample guarantees. To overcome the curse of dimensionality and to avoid the risk and/or expense in online sampling, we will  incorporate our critic algorithm with linear function approximation and off-policy learning. When both function approximation and off-policy learning are employed in TD-learning (which is essentially a bootstrapped stochastic iterative algorithm), the resulting algorithm can be unstable. As a result, function approximation, off-policy learning, and bootstrapping together are characterized as the ``deadly triad'' \citep{sutton2018reinforcement}. To overcome the deadly triad in the critic, we design a single time-scale TD-learning algorithm that uses multi-step return and generalized importance sampling factors. The proposed algorithm is provably stable and does not suffer from the high variance from using importance sampling. In addition, the mean-square error achieves the optimal $\mathcal{O}(1/k)$ convergence rate.

\textit{Establishment of the $\tilde{\mathcal{O}}(\epsilon^{-2})$ Sample Complexity for General Policy-Based Methods.}
After combining the geometric convergence of the actor and the $\mathcal{O}(1/k)$ convergence of the critic, we obtain an overall $\tilde{\mathcal{O}}(1/\epsilon^2)$ sample complexity for a general policy-based algorithm (which includes natural actor-critic as a special case) under off-policy sampling and linear function approximation. The sample complexity matches with that of typical value-based algorithms such as $Q$-learning \citep{li2020sample}. Even in the tabular setting, which is a special case of linear function approximation, this is the first time that an $\tilde{\mathcal{O}}(\epsilon^{-2})$ sample complexity is achieved for off-policy natural actor-critic, thereby advancing previous art in the literature.

\subsection{Related Literature}\label{subsec:literature}

Analogous to VI and PI for solving MDPs, RL algorithms can be divided into two categories: value-based method and policy-based method.
Popular value-space methods include $Q$-learning \citep{watkins1992q} and variants of TD-learning \citep{sutton1988learning}, both of which achieve the $\tilde{\mathcal{O}}(\epsilon^{-2})$ sample complexity \citep{li2020sample,chen2021finite,qu2020finite}. Popular policy-based methods include  actor-critic \citep{konda2000actor}, its variant natural actor-critic \citep{kakade2001natural}, and API \citep{bertsekas2011approximate}. A policy-based algorithm usually consists of an actor for policy improvement and a critic for policy evaluation. 

\textbf{TD-Learning.} The policy evaluation sub-problem is usually solved with TD-learning and its variants \citep{sutton1988learning}. The asymptotic convergence of TD-learning was established in \cite{tsitsiklis1994asynchronous,dayan1994td,bertsekas2009projected}. Finite-sample analysis of various TD-learning algorithms using on-policy sampling was performed in \cite{chen2021finite}, and using off-policy sampling in \cite{khodadadian2021finite,chen2021off}. In the function approximation setting, TD-learning with linear function approximation was studied in \cite{tsitsiklis1997analysis,lazaric2012finite,srikant2019finite,bhandari2018finite} when using on-policy sampling. In the off-policy linear function approximation setting, due to the presence of the deadly triad, TD-learning algorithms can diverge \citep{sutton2018reinforcement}. Variants of TD-learning algorithms such as TDC \citep{sutton2009fast}, GTD \citep{sutton2008convergent}, emphathic TD \citep{sutton2016emphatic}, and $n$-step TD (with a large enough $n$) \citep{chen2021NACLFA} were used to resolve the divergence issue, and the finite-sample bounds were studied in \cite{ma2020variance,wang2021finite,chen2021NACLFA}. Note that TDC, GTD, and emphatic TD are two time-scale algorithms, while vanilla $n$-step TD is single time-scale, it suffers from a high variance due to the cumulative product of the importance sampling factors. See Appendix \ref{ap:compare} of this work for a detailed discussion.

\textbf{(Natural) Policy Gradient.} The policy gradient method was proposed and was shown to converge in \cite{sutton1999policy,baxter2001infinite,agarwal2019theory}. natural policy gradient, proposed in \cite{kakade2001natural}, is a variant of policy gradient method where the inverse of the fisher information matrix was used as a pre-conditioner. The $O(1/k)$ convergence of natural policy gradient was shown in \cite{agarwal2019theory}. Later, by introducing regularization, geometric convergence of natural policy gradient was shown in \cite{lan2021policy,cen2021fast,cayci2021linear}. A recent paper \cite{xiao2022convergence} \footnote{\cite{xiao2022convergence} is subsequent to this work.} shows geometric convergence of natural policy gradient without regularization. Beyond geometric convergence, \cite{khodadadian2021linear} shows that asymptotically natural policy gradient converges at a super-linear rate. However, the natural policy gradient algorithm studied in \cite{khodadadian2021linear} uses adaptive stepsizes that depend on model parameters, and hence cannot be used in an actor-critic framework where the model is unknown to the agent. In this paper, we take a different perspective and view natural policy gradient as an approximate version of PI, which enables us to show the geometric convergence of natural policy gradient using only the properties (i.e., contraction and monotonicity) of the Bellman operator. Importantly, we do not require regularization, and our result can be directly used to establish an overall $\mathcal{O}(\epsilon^{-2})$ sample complexity of natural actor-critic. 

\textbf{Approximate Policy Iteration.} PI has been a popular method to solve MDPs \citep{puterman1995markov}. In the RL setting, PI has to be implemented in an approximate manner due to the lack of knowledge about the environmental model. The convergence rate and asymptotic error bound of API have been studied in \cite{scherrer2014approximate,munos2003error}. A convergent form of API with Lipschitz continuous policy update was proposed in \cite{perkins2002convergent}. Other variants of API such as least square PI and rollout sampling API schemes were proposed and studied in \cite{lagoudakis2003least} and \cite{dimitrakakis2008rollout}, respectively. See \cite{bertsekas2011approximate,powell2011review} for detailed surveys about API methods.

\textbf{(Natural) Actor-Critic.} 
The asymptotic convergence of on-policy actor-critic was established in \cite{williams1990mathematical, borkar2009stochastic, borkar1997actor} when using a tabular representation, and in \cite{konda2000actor, bhatnagar2009natural} when using function approximation. In recent years, there has been an increasing interest in understanding the finite-sample behavior of (natural) actor-critic algorithms. Here is a non-exhaustive list: \cite{lan2021policy, khodadadian2021finite,zhang2019convergence, qiu2019finite, kumar2019sample, liu2019neural, wang2019neural, liu2020improved, wu2020finite,cayci2021linear}. The state-of-the-art sample complexity of on-policy natural actor-critic is $\tilde{\mathcal{O}}(\epsilon^{-2})$ \cite{lan2021policy}. However, only tabular RL was considered in \cite{lan2021policy}. In the off-policy setting, finite-sample analysis of natural actor-critic was studied in \cite{chen2021NAC} when using a tabular representation, and in \cite{chen2021NACLFA} when using linear function approximation, and the sample complexity in both cases is $\tilde{\mathcal{O}}(\epsilon^{-3})$. In this work, we advance previous art by establishing an $\tilde{\mathcal{O}}(\epsilon^{-2})$ sample complexity of policy-based algorithms in the off-policy function approximation setting.

\subsection{Background on Reinforcement Learning}\label{subsec:background}
We consider modeling the RL problem as an infinite horizon MDP, which is defined in the following.
\begin{definition}
	An infinite horizon MDP is composed by a $5$-tuple $(\mathcal{S},\mathcal{A},\mathcal{P},\mathcal{R},\gamma)$, where $\mathcal{S}$ is the finite state-space, $\mathcal{A}$ is the finite action-space, $\mathcal{P}=\{P_a\in\mathbb{R}^{|\mathcal{S}|\times|\mathcal{S}}\mid a\in\mathcal{A}\}$ is a set of transition probability matrices, and $P_a(s,s')$ is the probability of going from state $s$ to state $s'$ under action $a$, $\mathcal{R}:\mathcal{S}\times\mathcal{A}\mapsto\mathbb{R}_+$ is the reward function, i.e., $\mathcal{R}(s,a)$ is the reward of taking action $a$ at state $s$, $\gamma\in (0,1)$ is the discount factor, which captures how much weight we assign to future reward.
\end{definition}

We assume without loss of generality that $\max_{s,a}\mathcal{R}(s,a)\leq 1$. The goal is to find an optimal policy $\pi^*$ of selecting actions so that the long term reward is maximized. Formally, define the state-action value function associated with a policy $\pi$ at state-action pair $(s,a)$ as $Q^\pi(s,a)=\mathbb{E}_\pi\left[\sum_{k=0}^\infty \gamma^k\mathcal{R}(S_k,A_k)\;\middle|\; S_0=s,A_0=a\right]$,
where we use the notation $\mathbb{E}_\pi[\,\cdot\,]$ to indicate that the actions are chosen according to the policy $\pi$. Then the goal is to find an optimal policy $\pi^*$ such that $Q^*:=Q^{\pi^*}$ is maximized uniformly for all $(s,a)$. A popular approach to solve the RL problem is to use policy-based methods, where the agent iteratively performs a critic step to estimate the value function of the current policy iterate, and an actor step to update the policy. 

\section{The Actor for Policy Improvement}\label{sec:actor}
In this section we study various policy-based algorithms (including natural policy gradient) and establish their convergence rates.

\subsection{Policy Update Rules}\label{subsec:rules}
We begin by presenting a generic policy-based algorithm in the following. 
\begin{algorithm}[H]\caption{A Generic Policy-Based Algorithm}\label{algorithm:API}
	\begin{algorithmic}[1] 
		\STATE {\bfseries Input:} Integer $T$ and initial policy $\pi_0(\cdot\mid s)\sim \text{Unif}(\mathcal{A})$ for all $s\in\mathcal{S}$
		\FOR{$t=0,1,\dots,T-1$}
		\STATE Critic estimates $Q^{\pi_t}$, and outputs $Q_t$\hfill{$\triangleright$ Policy Evaluation}
		\STATE Actor updates the policy according to $\pi_{t+1}=G(Q_t,\pi_t)$ \hfill $\triangleright$ Policy Improvement
		\ENDFOR
		\STATE\textbf{Output:} $\pi_{T}$
	\end{algorithmic}
\end{algorithm}

In Algorithm \ref{algorithm:API}, the function $G(\cdot,\cdot)$ represents the policy update rule, which takes the current policy iterate $\pi_t$ and the $Q$-function estimate $Q_t$ as inputs, and outputs the next policy $\pi_{t+1}$. Many existing policy update rules fit into this framework, as elaborated below.

\paragraph{Natural Policy Gradient.} A popular approach for updating the policy is to use the natural policy gradient, which can be viewed as a variant of policy gradient where the fisher information matrix is introduced as a pre-conditioner. Alternatively, natural policy gradient is a mirror descent update with the Bregman divergence being replaced by the $\mathcal{K}\mathcal{L}$ divergence. See \cite{agarwal2019theory} for more details. Mathematically, natural policy gradient updates the policy in an multiplicative manner according to
\begin{align}\label{eq:natural policy gradient_policy}
	\pi_{t+1}(a|s)=\frac{\pi_t(a|s)\exp(\beta_t Q_t(s,a))}{\sum_{a'\in\mathcal{A}}\pi_t(a'|s)\exp(\beta_t Q_t(s,a))},\quad \forall \;(s,a),
\end{align}
where $\{\beta_t\}$ is a positive sequence of real numbers, and $Q_t$ is the estimate of the value function $Q^{\pi_t}$. Note that the next policy $\pi_{t+1}$ is uniquely determined by the previous policy $\pi_t$ and its $Q$-function estimate $Q_t$. Therefore, natural policy gradient is a special case of Algorithm \ref{algorithm:API}.

\paragraph{Approximate Policy Iteration.}
A classical approach for policy improvement is to use API. In practice, the following two policy update rules are frequently used:
\begin{align*}
	\text{(1)}\quad\quad \pi_{t+1}(a|s)=\;&\frac{\exp(\beta_t Q_t(s,a))}{\sum_{a'\in\mathcal{A}}\exp(\beta_t Q_t(s,a'))},\quad \forall\;(s,a). \tag{Boltzmann Softmax Update}\\
	\text{(2)} \quad\quad \pi_{t+1}(a|s)=\;&\begin{dcases}
		\frac{\beta_t}{|\mathcal{A}|},&a\neq \arg\max_{a'\in\mathcal{A}}Q_t(s,a'),\\
		\frac{\beta_t}{|\mathcal{A}|}+1-\beta_t,&a= \arg\max_{a'\in\mathcal{A}}Q_t(s,a'),\tag{$\epsilon$-Greedy Update}
	\end{dcases}
\end{align*}
where $\{\beta_t\}$ in either case is a positive sequence of real numbers\footnote{While the update rule in (2) does not involve $\epsilon$, we use the terminology ``$\epsilon$-greedy'' for consistency with existing literature.}. In (2), when the maximizer $\arg\max_{a'\in\mathcal{A}}Q_t(s,a')$ is not unique, we break tie with an arbitrary but fixed tie-breaking rule. For simplicity of notation, we will denote $a_{t,s}=\arg\max_{a'\in\mathcal{A}}Q_t(s,a')$. Unlike natural policy gradient, in both of the API update rules, the next policy $\pi_{t+1}$ depends only on $Q_t$ and not on $\pi_t$. This is important when we use function approximation, in which case as long as the $Q$-function is parametrized, we do not need to parametrize the policy.

\subsection{Finite-Time Analysis}\label{subsec:bounds_API}
In this section, we present the convergence rate analysis of Algorithm \ref{algorithm:API} for using either natural policy gradient update or API update. We first specify how to choose $\{\beta_t\}$, which we view as stepsizes. In particular, we consider using either constant stepsize, or geometrically increasing stepsizes. Let $\beta>0$ be a tunable parameter.

\begin{condition}[Constant Stepsize]\label{condition:stepsize}
	The sequence $\{\beta_t\}$ satisfies the following requirements.
	\begin{enumerate}[(1)]
		\item \textit{Natural Policy Gradient:} The parameter $\beta_t$ satisfies $\beta_t\geq \gamma \beta\log(1/\min_{s\in\mathcal{S}}\pi_t(a_{t,s}|s))$ for all $t$.
		\item \textit{Boltzmann Softmax Update:} The parameter $\beta_t$ satisfies $\beta_t\geq \gamma \beta\log(|\mathcal{A}|)$ for all $t\geq 0$.
		\item \textit{$\epsilon$-Greedy Update:} The parameter $\beta_t$ is state-dependent, and satisfies $\beta_{t,s}\geq 2\gamma\beta \max_{a\in\mathcal{A}}|Q_t(s,a)|$ for all $s$ and $t$.
	\end{enumerate}
\end{condition}

\begin{condition}[Geometrically Increasing Stepsizes]\label{condition:increasing_stepsize}
	The sequence $\{\beta_t\}$ satisfies the following requirements.
	\begin{enumerate}[(1)]
		\item \textit{Natural Policy Gradient:} The parameter $\beta_t$ satisfies $\beta_t\geq \log(\frac{1}{\min_{s\in\mathcal{S}}\pi_t(a_{t,s}|s)})/\gamma^{2t-1}$ for all $t$.
		\item \textit{Boltzmann Softmax Update:} The parameter $\beta_t$ satisfies $\beta_t\geq \log(|\mathcal{A}|)/\gamma^{2t-1}$ for all $t$.
		\item \textit{$\epsilon$-Greedy Update:} The parameter $\beta_t$ is state-dependent, and satisfies $\beta_{t,s}\geq 2\max_{a\in\mathcal{A}}|Q_t(s,a)|/\gamma^{2t-1}$ for all $s$ and $t$.
	\end{enumerate}
\end{condition}

Next, we state the result for the finite-time analysis of Algorithm \ref{algorithm:API}.

\begin{theorem}\label{thm:main}
	Consider $\{\pi_t\}$ generated by Algorithm \ref{algorithm:API} for using either natural policy gradient update or API update. 
	\begin{enumerate}[(1)]
		\item Suppose that $\{\beta_t\}$ is chosen according to Condition \ref{condition:stepsize}. Then we have 
		\begin{align}\label{bound:constant}
			\mathbb{E}[\|Q^*-Q^{\pi_T}\|_\infty]\leq\;& \underbrace{\gamma^T\|Q^*-Q^{\pi_0}\|_\infty}_{N_1}+\underbrace{\frac{2\gamma}{1-\gamma}\sum_{t=0}^{T-1}\gamma^{T-1-t}\mathbb{E}[\|Q^{\pi_t}-Q_t\|_\infty]}_{N_2}+\underbrace{\frac{2\gamma}{\beta (1-\gamma)^2}}_{N_3}.
		\end{align}
		\item Suppose that $\{\beta_t\}$ is chosen according to Condition \ref{condition:increasing_stepsize}. Then we have 
		\begin{align}\label{bound:increasing}
			\mathbb{E}[\|Q^*-Q^{\pi_T}\|_\infty]\leq\;&N_1+N_2+\underbrace{\frac{2\gamma^T}{(1-\gamma)^2}}_{N_3'}.
		\end{align}

	\end{enumerate}
\end{theorem}

On the RHS of Eq. (\ref{bound:constant}), the term $N_1$ represents the convergence bias of the actor, and goes to zero geometrically fast as $T$ goes to infinity. The term $N_2$ represents the error in the critic, and vanishes in the MDP tabular setting where we can (in principle) exactly compute $Q^{\pi_t}$. In the RL with function approximation setting, the term $N_2$ involves the stochastic error due to sampling and the error due to function approximation, both of which will be exactly characterized in Section \ref{sec:API}, where we present the sample complexity of actor-critic. The term $N_3$ captures the error introduced to the algorithm by the policy update rule $G(\cdot,\cdot)$. To elaborate, recall that for a  discounted MDP there always exists a deterministic optimal policy. Suppose that the optimal policy is unique, and we use $\epsilon$-greedy update in Algorithm \ref{algorithm:API} line 4 with a constant stepsize; see Condition \ref{condition:stepsize}. Then we can never truly find the optimal policy $\pi^*$ because of the deterministic nature of $\pi^*$ and the stochastic nature of our policy iterates $\{\pi_t\}$. As a result, the difference between $Q^*$ and $Q^{\pi_t}$ will always be above some threshold, which depends on the choice of $\beta$ in Condition \ref{condition:stepsize}, and is captured by $N_3$. Observe that $N_3$ can be made arbitrarily small by using large enough $\beta$. Alternatively, we can use geometrically increasing stepsizes as suggested in Condition \ref{condition:increasing_stepsize}, in which case the term $N_3'$ in Eq. (\ref{bound:increasing}) goes to zero at a geometric rate.

\subsection{Geometric Convergence of Natural Policy Gradient}
Theorem \ref{thm:main} implies the geometric convergence of natural policy gradient, due to the popularity of which, we present the result as a corollary in the following.

\begin{algorithm}[H]\caption{Natural Policy Gradient under Softmax Policy and Tabular Representation}\label{algorithm:natural policy gradient}
	\begin{algorithmic}[1] 
		\STATE {\bfseries Input:} Integer $T$ and initial policy $\pi_0(\cdot|s)\sim \text{Unif}(\mathcal{A})$ for all $s\in\mathcal{S}$
		\FOR{$t=0,1,\dots,T-1$}
		\STATE $\pi_{t+1}(a|s)=\frac{\pi_t(a|s)\exp(\beta_t Q^{\pi_t}(s,a))}{\sum_{a'\in\mathcal{A}}\pi_t(a'|s)\exp(\beta_t Q^{\pi_t}(s,a'))}$ for all $(s,a)$.
		\ENDFOR
		\STATE\textbf{Output:} $\pi_{T}$
	\end{algorithmic}
\end{algorithm}

\begin{corollary}\label{co:natural policy gradient}
	Consider $\{\pi_T\}$ generated by Algorithm \ref{algorithm:natural policy gradient}. Suppose that $\beta_t\geq \log(1/\min_{s\in\mathcal{S}}\pi_t(a_{t,s}\mid s))/\gamma^{2t-1}$ for all $t=0,\cdots,T-1$. Then we have $\mathbb{E}[\|Q^*-Q^{\pi_T}\|_\infty]\leq  \frac{3\gamma^T}{(1-\gamma)^2}$.
\end{corollary}

The proof of Corollary \ref{co:natural policy gradient} follows directly from Theorem \ref{thm:main} Eq. (\ref{bound:increasing}) and the following two observations: (1) $N_2$ vanishes as we directly use $Q^{\pi_t}$ for updating the policy in Algorithm \ref{algorithm:natural policy gradient}, and (2) $N_1\leq \frac{\gamma^T}{1-\gamma}$ since $\|Q^*-Q^{\pi_0}\|_\infty\leq \frac{1}{1-\gamma}$.

\subsection{Proof Sketch of Theorem \ref{thm:main}}\label{subsec:theorem1sketch}

In this section, we present the high level ideas in proving Theorem \ref{thm:main}. The detailed proof is deferred to Section \ref{subsec:pf:actor}. We first introduce some notation. Let $\mathcal{H}:\mathbb{R}^{|\mathcal{S}||\mathcal{A}|}\mapsto\mathbb{R}^{|\mathcal{S}||\mathcal{A}|}$ be the Bellman optimality operator defined as 
\begin{align*}
	[\mathcal{H}(Q)](s,a)=\mathcal{R}(s,a)+\gamma\mathbb{E}\left[\max_{a'\in\mathcal{A}}Q(S_{k+1},a')\;\middle|\; S_k=s,A_k=a\right],\quad \forall\;Q\in\mathbb{R}^{|\mathcal{S}||\mathcal{A}|},\;\forall\;(s,a),
\end{align*}
and let $\mathcal{H}_\pi:\mathbb{R}^{|\mathcal{S}||\mathcal{A}|}\mapsto\mathbb{R}^{|\mathcal{S}||\mathcal{A}|}$ be the Bellman operator associated with some policy $\pi$ defined as
\begin{align*}
	[\mathcal{H}_\pi(Q)](s,a)=\mathcal{R}(s,a)+\gamma\mathbb{E}_\pi[Q(S_{k+1},A_{k+1})\mid S_k=s,A_k=a],\quad \forall\;Q\in\mathbb{R}^{|\mathcal{S}||\mathcal{A}|},\;\forall\;(s,a).
\end{align*}

The proof is divided into two steps. In the first step we bound the difference between $Q^{\pi_T}$ and $Q^*$ in terms of $\|Q^{\pi_t}-Q_t\|_\infty$ and $\|\mathcal{H}_{\pi_{t+1}}(Q_t)-\mathcal{H}(Q_t)\|_\infty$, $t=0,\cdots,T-1$. The term $\|Q^{\pi_t}-Q_t\|_\infty$ can be viewed as the critic error and the term $\|\mathcal{H}_{\pi_{t+1}}(Q_t)-\mathcal{H}(Q_t)\|_\infty$ can be viewed as the actor error (for not using PI). The proof technique in this step was inspired by \cite[Section 6.2]{bertsekas1996neuro}. However, only asymptotic error bound of API was established in  \cite{bertsekas1996neuro}, while we establish finite-sample bounds for various policy update rules. In the second step, we further bound $\|\mathcal{H}_{\pi_{t+1}}(Q_t)-\mathcal{H}(Q_t)\|_\infty$ using our conditions (cf. Conditions \ref{condition:stepsize} and \ref{condition:increasing_stepsize}) on choosing the stepsizes $\{\beta_t\}$.

\section{The Critic for Policy Evaluation}\label{sec:PE}
So far we have been focusing on the analysis of the actor. In this section, we switch our focus to the critic, i.e., how to obtain an estimate of $Q^{\pi_t}$ in Algorithm \ref{algorithm:API} line 3.

Consider estimating the $Q$-function $Q^\pi$ of a given target policy $\pi$ using TD-learning. Depending on whether the policy $\pi_b$ used to collect samples (called the behavior policy) is equal to the target policy $\pi$ or not, there are on-policy TD-learning (i.e., $\pi_b= \pi$) and off-policy TD-learning (i.e., $\pi_b\neq  \pi$). Compared to on-policy sampling, off-policy sampling is sometimes more preferred in practice as it does not require collecting new samples, which often come at a risk and/or cost. In addition, off-policy sampling enables data reuse so that the agent can learn in an off-line manner using historical data.

It is known that TD-learning becomes computationally intractable when the size of the state-action space is large, which is referred to as the curse of dimensionality. This motivates the use of function approximation. In linear function approximation, we choose a set of basis vectors $\phi_i\in\mathbb{R}^{|\mathcal{S}||\mathcal{A}|}$, $1\leq i\leq d$, and try to approximate the target value function $Q^\pi$ using linear combinations of the basis vectors. Specifically, let $\Phi\in\mathbb{R}^{|\mathcal{S}||\mathcal{A}|\times d}$ be a matrix defined by $\Phi=[\phi_1,\cdots,\phi_d]$, and let $\phi(s,a)=[\phi_1(s,a),\phi_2(s,a),\cdots, \phi_d(s,a)]^\top\in\mathbb{R}^d$ be the $(s,a)$-th row of the matrix $\Phi$, which can be viewed as the feature vector associated with the state-action pair $(s,a)$.  Then, the goal is to find from the linear sub-space $\mathcal{Q}=\{\Tilde{Q}_w=\Phi w \mid w\in\mathbb{R}^d\}$ the ``best'' approximation of the $Q$-function $Q^\pi$, where $w\in\mathbb{R}^d$ is the weight vector.

When TD-learning is used along with off-policy sampling and linear function approximation, the deadly triad is formed and the algorithm can be unstable.  We next propose a generic framework of TD-learning algorithms (including two specific algorithms: the $\lambda$-averaged $Q$-trace and the two-sided $Q$-trace), which provably converge in the presence of the deadly triad, and do not suffer from the high variance issue in off-policy learning. Since we work with MDPs with finite state-action spaces, we assume without loss of generality that the matrix $\Phi$ has linearly independent columns, and is normalized so that $\|\Phi\|_\infty=\max_{s,a}\|\phi(s,a)\|_1\leq 1$.

\subsection{Algorithm Design}\label{subsec:PE_algorithm}

We present in Algorithm \ref{algorithm:Off-Policy-TD} a generic TD-learning algorithm using off-policy sampling and linear function approximation. The two most important steps in Algorithm \ref{algorithm:Off-Policy-TD} are line 3, where the temporal differences are computed, and line 4, where we update the weight vector $w_t$ using discounted multi-step return. Both steps involve the use of the generalized importance sampling factors $c(\cdot,\cdot)$ and $\rho(\cdot,\cdot)$, the choices of which are of vital importance to the behavior of the algorithm. We next present two specific choices, resulting in two novel algorithms called $\lambda$-averaged $Q$-trace and two-sided $Q$-trace. 

\begin{algorithm*}[ht]\caption{A Generic Multi-Step Off-Policy TD-Learning with Linear Function Approximation}\label{algorithm:Off-Policy-TD}
	\begin{algorithmic}[1] 
		\STATE \textbf{Input}: Integer $K$, bootstrapping parameter $n$, stepsize sequence $\{\alpha_k\}$, initialization $w_0$, target policy $\pi$, behavior policy $\pi_b$, generalized importance sampling factors $c,\rho:\mathcal{S}\times\mathcal{A}\mapsto\mathbb{R}_+$, and a single trajectory of samples $\{(S_k,A_k)\}_{0\leq k\leq K+n-1}$ generated by the behavior policy $\pi_b$.
		\FOR{$k=0,1,\cdots,K-1$}
		\STATE $\Delta_i(w_k)=\mathcal{R}(S_i,A_i)+\gamma \rho(S_{i+1},A_{i+1})\phi(S_{i+1},A_{i+1})^\top w_k-\phi(S_{i},A_{i})^\top w_k$, for all $i\in \{k,k+1,\cdots,k+n-1\}$
		\STATE $w_{k+1}=
		w_k+\alpha_k\phi(S_k,A_k)\sum_{i=k}^{k+n-1}\gamma^{i-k}\prod_{j=k+1}^ic(S_j,A_j)\Delta_i(w_k)$ 
		\ENDFOR
		\STATE\textbf{Output:} $w_K$
	\end{algorithmic}
\end{algorithm*}

\textit{The $\lambda$-Averaged $Q$-Trace Algorithm.}
Let $\lambda\in\mathbb{R}^{|\mathcal{S}|}$ be a vector-valued tunable parameter satisfying $\lambda\in [\bm{0},\bm{1}]$. Then the generalized importance sampling factors are chosen as $c(s,a)=\rho(s,a)=\lambda(s)\frac{\pi(a|s)}{\pi_b(a|s)}+1-\lambda(s)$
for all $(s,a)$. Observe that when $\lambda=\bm{1}$, we have $c(s,a)=\rho(s,a)=\frac{\pi(a|s)}{\pi_b(a|s)}$, and Algorithm \ref{algorithm:Off-Policy-TD} reduces to the convergent multi-step off-policy TD-learning algorithm presented in \cite{chen2021NACLFA}, which however suffers from an exponential large variance due to the cumulative product of the importance sampling factors. See Appendix \ref{ap:compare} for more details. On the other hand, when $\lambda=\bm{0}$, we have $c(s,a)=\rho(s,a)=1$, and hence the product of the generalized importance sampling factors is deterministically equal to one, resulting in no variance at all. However in this case, we are essentially performing policy evaluation of the behavior policy $\pi_b$ instead of the target policy $\pi$, hence there will be an asymptotic bias in the limit of Algorithm \ref{algorithm:Off-Policy-TD}. More generally, when $\lambda\in (\bm{0},\bm{1})$, there is a trade-off between the variance in the stochastic iterates $\{w_k\}$ and the bias in the limit point. Such trade-off will be studied quantitatively in Section \ref{subsec:PE_bounds}.

\textit{The Two-Sided $Q$-Trace Algorithm.}
To introduce the algorithm, we first define the two-sided truncation function. Given upper and lower truncation levels $a,b\in\mathbb{R}$ satisfying $0<a<b$, we define $g_{a,b}:\mathbb{R}\mapsto\mathbb{R}$ as $g_{a,b}(x)=a$ when $x<a$, $g_{a,b}(x)=x$ when $a\leq x\leq b$, and $g_{a,b}(x)=b$ when $x>b$. Let $\ell,u\in\mathbb{R}^{|\mathcal{S}|}$ be two vector-valued tunable parameters satisfying $\bm{0}\leq \ell\leq \bm{1}\leq u$. Then, for the two-sided $Q$-trace algorithm, the generalized importance sampling factors are chosen as  $c(s,a)=\rho(s,a)=g_{\ell(s),u(s)}\left(\frac{\pi(a|s)}{\pi_b(a|s)}\right)$
for all $(s,a)$. The idea of truncating the importance sampling factors from above was already employed in existing algorithms such as Retrace$(\lambda)$ \citep{munos2016safe}, $V$-trace \citep{espeholt2018impala}, and $Q$-trace \citep{khodadadian2021finite}, and is used to control the high variance in off-policy learning. However, none of them were shown to converge in the function approximation setting. The main reason is that truncating the importance sampling factors from above undermines the impact of using multi-step return, which is crucial for us to overcome the deadly triad. Therefore, we introduce the lower truncation level as a compensation to maintain the degree of bootstrapping. This will be illustrated in detail in Section \ref{subsec:PE_bounds}.

\subsection{The Generalized Projected Bellman Equation}\label{subsec:PE_motivation}

We next theoretically analyze Algorithm \ref{algorithm:Off-Policy-TD}. Specifically, in this section, we formulate Algorithm \ref{algorithm:Off-Policy-TD} as a stochastic approximation algorithm for solving a generalized projected Bellman equation (PBE) and study its properties. We begin by stating our assumption.

\begin{assumption}\label{as:MC}
	The behavior policy $\pi_b$ satisfies $\pi_b(a|s)>0$ for all $(s,a)$, and induces an irreducible and aperiodic Markov chain $\{S_k\}$.
\end{assumption}

Assumption \ref{as:MC} was commonly imposed in existing work \citep{tsitsiklis1997analysis,bhandari2018finite} to ensure that the behavior policy has sufficient exploration, which is known to be a necessary component for learning.
Under Assumption \ref{as:MC}, the Markov chain $\{S_k\}$ induced by $\pi_b$ has a unique stationary distribution  $\mu\in\Delta^{|\mathcal{S}|}$. Moreover, there exist $C\geq 1$ and $\sigma\in (0,1)$ such that $\max_{s\in\mathcal{S}}\|P^k_{\pi_b}(s,\cdot)-\mu(\cdot)\|_{\text{TV}}\leq C\sigma^k$ for all $k\geq 0$, where $P_{\pi_b}$ is the transition probability matrix of the Markov chain $\{S_k\}$ under $\pi_b$ \citep{levin2017markov}. 

For simplicity of notation, denote $c_{i,j}=\prod_{k=i}^jc(S_k,A_k)$. Algorithm \ref{algorithm:Off-Policy-TD} can be viewed as a stochastic iterative algorithm for solving the following system of equations (in terms of $w$):
\begin{align}\label{eq:11}
	\mathbb{E}_{S_0\sim \mu}\left[\phi(S_0,A_0)\sum_{i=0}^{n-1}\gamma^ic_{1,i}\Delta_i(w)\right]=\bm{0},
\end{align}
where $A_i\sim \pi_b(\cdot|S_i)$ and $S_{i+1}\sim P_{A_i}(S_i,\cdot)$. The following lemma formulates Eq. (\ref{eq:11}) in the form of a generalized PBE. 

To present the lemma, we need to introduce more notation. Let $\mathcal{K}_{SA}\in\mathbb{R}^{|\mathcal{S}||\mathcal{A}|\times |\mathcal{S}||\mathcal{A}|}$ be a diagonal matrix with diagonal entries $\{\mu(s)\pi_b(a|s)\}_{(s,a)\in\mathcal{S}\times\mathcal{A}}$, and let $\mathcal{K}_{SA,\min}$ be the minimal diagonal entry. Let $\|\cdot\|_{\mathcal{K}_{SA}}$ be the weighted $\ell_2$-norm with weights $\{\mu(s)\pi_b(a|s)\}_{(s,a)\in\mathcal{S}\times\mathcal{A}}$, and denote $\text{Proj}_{\mathcal{Q}}$ as the projection operator onto the linear sub-space $\mathcal{Q}$ with respect to $\|\cdot\|_{\mathcal{K}_{SA}}$. Let $\mathcal{T}_c,\mathcal{H}_\rho:\mathbb{R}^{|\mathcal{S}||\mathcal{A}|}\mapsto\mathbb{R}^{|\mathcal{S}||\mathcal{A}|}$ be two operators defined as
\begin{align*}
	[\mathcal{T}_c(Q)](s,a)&=\sum_{i=0}^{n-1}\mathbb{E}_{\pi_b}[\gamma^ic_{1,i}Q(S_i,A_i)\mid  S_0=s,A_0=a]\\
	[\mathcal{H}_\rho(Q)](s,a)&=\mathcal{R}(s,a)+\gamma\mathbb{E}_{\pi_b}[\rho(S_{1},A_{1})Q(S_{1},A_{1})\mid S_0=s,A_0=a]
\end{align*}
for any $Q\in\mathbb{R}^{|\mathcal{S}||\mathcal{A}|}$ and state-action pair $(s,a)$. 
\begin{lemma}\label{le:equation}
	Eq. (\ref{eq:11}) is equivalent to:
	\begin{align}\label{eq:pbe-key}
		\Phi w=\text{Proj}_{\mathcal{Q}} \mathcal{B}_{c,\rho}(\Phi w),
	\end{align}
	where $\mathcal{B}_{c,\rho}(\cdot)$ is the generalized Bellman operator defined as $\mathcal{B}_{c,\rho}(Q)=\mathcal{T}_c(\mathcal{H}_\rho(Q)-Q)+Q$ for any $Q\in\mathbb{R}^{|\mathcal{S}||\mathcal{A}|}$.
\end{lemma}

The generalized Bellman operator $\mathcal{B}_{c,\rho}(\cdot)$ was previously introduced in \cite{chen2021off} to study off-policy TD-learning algorithms in the \textit{tabular} setting (i.e., $\Phi=I_{SA}$), where the contraction property of $\mathcal{B}_{c,\rho}(\cdot)$ (as well as its asynchronous variant) was shown. However, $\mathcal{B}_{c,\rho}(\cdot)$ alone being a contraction is not enough to guarantee the convergence of Algorithm \ref{algorithm:Off-Policy-TD} because of the use of function approximation, which adds an additional projection operator $\text{Proj}_{\mathcal{Q}}$. What we truly need for the stability and accuracy of Algorithm \ref{algorithm:Off-Policy-TD} is that (1) the composed operator $\text{Proj}_{\mathcal{Q}} \mathcal{B}_{c,\rho}(\cdot)$ is a contraction mapping, and (2) the solution $w_{c,\rho}^\pi$ of Eq. (\ref{eq:pbe-key}) is such that $\Phi w_{c,\rho}^\pi$ is a valid approximation of the $Q$-function $Q^\pi$. We next provide sufficient conditions on the choices of the generalized importance sampling factors $c(\cdot,\cdot)$ and $\rho(\cdot,\cdot)$, and the bootstrapping parameter $n$ so that the above two requirements are satisfied. 

Let $D_c,D_\rho\in\mathbb{R}^{|\mathcal{S}||\mathcal{A}|\times|\mathcal{S}||\mathcal{A}|}$ be two diagonal matrices such that $D_c((s,a),(s,a))=\sum_{a'\in\mathcal{A}}\pi_b(a'|s)c(s,a')$ and $D_\rho((s,a),(s,a))=\sum_{a'\in\mathcal{A}}\pi_b(a'|s)\rho(s,a')$ for all $(s,a)$. Let $D_{c,\max}$ and $D_{\rho,\max}$ ($D_{c,\min}$ and $D_{\rho,\min}$) be the maximam (minimum) diagonal entries of the matrices $D_c$ and $D_\rho$, respectively. 

\begin{condition}\label{as:IS_ratio}
	The generalized importance sampling factors $c(\cdot,\cdot)$ and $\rho(\cdot,\cdot)$ satisfy (1) $c(s,a)\leq \rho(s,a)$ for all $(s,a)$, (2) $ D_{\rho,\max}<1/\gamma$, and (3) $\frac{\gamma (D_{\rho,\max}-D_{c,\min})}{(1-\gamma D_{c,\min})\sqrt{\mathcal{K}_{SA,\min}}}<1$.
\end{condition}

Condition \ref{as:IS_ratio} (1) and (2) were previously introduced in \cite{chen2021off}, and were used to show the contraction property of the operator $\mathcal{B}_{c,\rho}(\cdot)$. In particular, it was shown that the generalized Bellman operator $\mathcal{B}_{c,\rho}(\cdot)$ is a contraction mapping with respect to $\|\cdot\|_\infty$, with contraction factor $\tilde{\gamma}(n)=1-f_n(\gamma D_{c,\min})(1-\gamma D_{\rho,\max})$, where $f_n:\mathbb{R}\mapsto\mathbb{R}$ is defined as $f_n(x)=\sum_{i=0}^{n-1}x^i$ for any $x\geq 0$. It is clear that $\tilde{\gamma}(n)\in (0,1)$, and is a decreasing function of $n$. 

As illustrated earlier, $\mathcal{B}_{c,\rho}(\cdot)$ being a contraction mapping is not sufficient to guarantee the stability of Algorithm \ref{algorithm:Off-Policy-TD}. We need the composed operator $\text{Proj}_{\mathcal{Q}}\mathcal{B}_{c,\rho}(\cdot)$ to be contraction mapping with appropriate choice of $n$. This is guaranteed by Condition \ref{as:IS_ratio} (3). To see this, first note that we have the following lemma, which is obtained by using the contraction property of $\mathcal{B}_{c,\rho}(\cdot)$ and the ``equivalence'' between norms in finite-dimensional spaces.

\begin{lemma}\label{le:Lipschitz_factor}
	Under Condition \ref{as:IS_ratio}, it holds for any $Q_1,Q_2\in\mathbb{R}^{|\mathcal{S}||\mathcal{A}|}$ that 
	\begin{align*}
		\|\text{Proj}_{\mathcal{Q}}\mathcal{B}_{c,\rho}(Q_1)-\text{Proj}_{\mathcal{Q}}\mathcal{B}_{c,\rho}(Q_2)\|_{\mathcal{K}_{SA}}
		\leq \frac{\tilde{\gamma}(n)}{\sqrt{\mathcal{K}_{SA,\min}}}\|Q_1-Q_2\|_{\mathcal{K}_{SA}}.
	\end{align*}
\end{lemma}

In view of Lemma \ref{le:Lipschitz_factor}, the composed operator  $\text{Proj}_{\mathcal{Q}}\mathcal{B}_{c,\rho}(\cdot)$ can be made a contraction mapping with an appropriate choice of $n$ as long as $\lim_{n\rightarrow\infty}\tilde{\gamma}(n)/\sqrt{\mathcal{K}_{SA,\min}}<1$, which after some algebra is equivalent to Condition \ref{as:IS_ratio} (3). 

We next show that the generalized importance sampling factors designed for the $\lambda$-averaged $Q$-trace and the two-sided $Q$-trace satisfy Condition \ref{as:IS_ratio} (3). First consider the $\lambda$-averaged $Q$-trace algorithm. Since both $D_c$ and $D_\rho$ are identity matrices (which implies $D_{\rho,\max}=D_{c,\min}=1$), Condition  \ref{as:IS_ratio} (3) is always satisfied. Next consider the two-sided $Q$-trace algorithm. For any choice of the upper truncation level $u\geq \bm{1}$, we can always choose the lower truncation level $\bm{0}\leq \ell\leq \bm{1}$ appropriately to satisfy Condition \ref{as:IS_ratio} (3). Specifically, for any $s\in\mathcal{S}$ and $u(s)\geq 1$, choosing $\ell(s)\leq 1$ such that $\sum_{a\in\mathcal{A}}\pi_b(a|s)g_{\ell(s),u(s)}(\pi(a|s)/\pi_b(a|s))=1$ satisfies Condition \ref{as:IS_ratio} (3). However, suppose we only truncate the importance sampling factors from above as in existing off-policy TD-learning algorithms such as Retrace$(\lambda)$ \citep{munos2016safe}, $Q^\pi(\lambda)$ \citep{harutyunyan2016q}, $V$-trace \citep{espeholt2018impala}, and $Q$-trace \citep{chen2021NAC}. It is not clear if we can make $D_{\rho,\max}$ and $D_{c,\min}$ arbitrarily close to satisfy Condition \ref{as:IS_ratio} (3). That is technical reason for us to introduce the lower truncation levels.

In the next lemma, we show that under Condition \ref{as:IS_ratio}, with properly chosen $n$, the composed operator $\text{Proj}_{\mathcal{Q}} \mathcal{B}_{c,\rho}(\cdot)$ is a contraction mapping, which ensures that Eq. (\ref{eq:pbe-key}) has a unique solution, denoted by $w_{c,\rho}^\pi$. Moreover, we provide performance guarantees on the solution $w_{c,\rho}^\pi$ in terms of an upper bound on the difference between the target value function $Q^\pi$ and our estimate $\Phi w_{c,\rho}^\pi$. Let $Q_{c,\rho}^\pi$ be the solution of generalized Bellman equation $Q=\mathcal{B}_{c,\rho}(Q)$, which is guaranteed to exist and is unique since $\mathcal{B}_{c,\rho}(\cdot)$ itself is a contraction mapping under Condition \ref{as:IS_ratio} (1) and (2) \citep{chen2021off}. 

\begin{lemma}\label{le:contraction}
	Under Condition \ref{as:IS_ratio}, suppose that the parameter $n$ is chosen such that $\gamma_c:=\tilde{\gamma}(n)/\sqrt{\mathcal{K}_{SA,\min}}<1$. Then the composed operator $\text{Proj}_{\mathcal{Q}} \mathcal{B}_{c,\rho}(\cdot)$ is a $\gamma_c$-contraction mapping with respect to $\|\cdot\|_{\mathcal{K}_{SA}}$. In this case, the unique solution $w_{c,\rho}^\pi$ of the generalized PBE (cf. Eq. (\ref{eq:pbe-key})) satisfies
	\begin{align}
		\|Q^\pi-\Phi w_{c,\rho}^\pi\|_{\mathcal{K}_{SA}}\leq\frac{\|Q_{c,\rho}^\pi-\text{Proj}_{\mathcal{Q}}Q_{c,\rho}^\pi\|_{\mathcal{K}_{SA}}}{\sqrt{1-\gamma_c^2}}+\frac{\gamma\max_{s\in\mathcal{S}}\|\pi(\cdot|s)-\pi_b(\cdot|s)\rho(s,\cdot)\|_1}{(1-\gamma)(1-\gamma D_{\rho,\max})}.\label{eq:Q_performance}
	\end{align}
\end{lemma}

The first term on the RHS of Eq. (\ref{eq:Q_performance}) captures the error due to function approximation, which is in the same spirit to Theorem 1 (4) of the seminal paper \cite{tsitsiklis1997analysis}, and vanishes in the tabular setting. The second term on the RHS of Eq. (\ref{eq:Q_performance}) arises because of the use of generalized importance sampling factors, which are introduced to overcome the high variance in  off-policy learning. Note that the second term vanishes when $\rho(s,a)=\pi(a|s)/\pi_b(a|s)$ for all $(s,a)$, which corresponds to choosing $\lambda=\bm{1}$ in $\lambda$-averaged $Q$-trace or choosing $\ell(s)\leq \min_{s,a}\pi(a|s)/\pi_b(a|s)$ and $u(s)\geq \max_{s,a}\pi(a|s)/\pi_b(a|s)$ for all $s$ in two-sided $Q$-trace. However, in these cases, the cumulative product of importance sampling factors leads to a high variance in Algorithm \ref{algorithm:Off-Policy-TD}. The trade-off between the variance and the bias in $w_{c,\rho}^\pi$ (i.e., second term on the RHS of Eq. (\ref{eq:Q_performance})) will be illustrated in detail in the next subsection where we present the finite-sample bound of Algorithm \ref{algorithm:Off-Policy-TD}.

\subsection{Finite-Sample Analysis}\label{subsec:PE_bounds}
With the contraction property of the generalized PBE established, when the stepsize sequence is non-summable but squared-summable, the almost sure convergence of Algorithm \ref{algorithm:Off-Policy-TD} directly follows from standard stochastic approximation results in the literature \citep{bertsekas1996neuro,borkar2009stochastic}.
In this section, we perform finite-sample analysis of Algorithm \ref{algorithm:Off-Policy-TD}. For ease of exposition, we here only present the finite-sample bounds of $\lambda$-averaged $Q$-trace and two-sided $Q$-trace, where $c(\cdot,\cdot)$ and $\rho(\cdot,\cdot)$ are explicitly specified. For the finite-sample guarantees of Algorithm \ref{algorithm:Off-Policy-TD} with more general choices of $c(\cdot,\cdot)$ and $\rho(\cdot,\cdot)$ (as long as Condition \ref{as:IS_ratio} is satisfied), please see Section \ref{subec:pf:critic}.

For any $\delta>0$, let $t_\delta=\min\{k\geq 0:\max_{s\in\mathcal{S}}\|P_{\pi_b}^k(s,\cdot)-\mu(\cdot)\|_{\text{TV}}\leq \delta\}$ be the mixing time of the Markov chain $\{S_k\}$ induced by $\pi_b$ with precision $\delta$. Note that Assumption \ref{as:MC} implies that $t_\delta=\mathcal{O}(\log(1/\delta))$. Let $\lambda_{\min}$ be the mininum eigenvalue of the positive definite matrix $\Phi^\top \mathcal{K}_{SA}\Phi$. Let $L=1+(\gamma \rho_{\max})^n$, where $\rho_{\max}=\max_{s,a}\rho(s,a)$. We next present finite-sample guarantees of the $\lambda$-averaged $Q$-trace algorithm.

\begin{theorem}\label{co:Qtrace-bound}
	Consider $\{w_k\}$ of the $\lambda$-averaged $Q$-trace Algorithm. Suppose that (1) Assumption \ref{as:MC} is satisfied, (2) $\lambda\in [\bm{0},\bm{1}]$, (3) the parameter $n$ is chosen such that $\gamma_c:=\gamma^n/\sqrt{\mathcal{K}_{SA,\min}}<1$. Then, when using constant stepsize $\alpha$ satisfying $\alpha(t_\alpha+n+1)\leq \frac{(1-\gamma_c)\lambda_{\min}}{130L^2}$, we have for all $k\geq t_\alpha+n+1$:
	\begin{align}
		\mathbb{E}[\|w_k-w_{c,\rho}^\pi\|_2^2]\leq c_1(1-(1-\gamma_c)\lambda_{\min}\alpha)^{k-(t_\alpha+n+1)}+c_2\frac{\alpha L^2(t_\alpha+n+1)}{(1-\gamma_c)\lambda_{\min}},\label{eq:Qtrace-bound}
	\end{align}
	where $c_1=(\|w_0\|_2+\|w_0-w_{c,\rho}^\pi\|_2+1)^2$ and $c_2=130(\|w_{c,\rho}^\pi\|_2+1)^2$.  
	Moreover, we have
	\begin{align}
		\|Q^\pi-\Phi w_{c,\rho}^\pi\|_{\mathcal{K}_{SA}}\leq\frac{\|Q_{c,\rho}^\pi-\text{Proj}_{\mathcal{Q}}Q_{c,\rho}^\pi\|_{\mathcal{K}_{SA}}}{\sqrt{1-\gamma_c^2}}+\frac{\gamma\max_{s\in\mathcal{S}}(1-\lambda(s))\|\pi(\cdot|s)-\pi_b(\cdot|s)\|_1}{(1-\gamma)^2}.\label{eq:Qtrace_performance}
	\end{align}
\end{theorem}

Following the common terminology in stochastic approximation literature, we call the first term on the RHS of Eq. (\ref{eq:Qtrace-bound}) \textit{convergence bias}, and the second term \textit{variance}. When constant stepsize is used, the convergence bias goes to zero at a geometric rate while the variance is a constant roughly proportional to $\alpha t_{\alpha}$. Since $\lim_{\alpha\rightarrow 0}\alpha t_\alpha=0$ under Assumption \ref{as:MC}, the variance can be made arbitrarily small by using small $\alpha$.

The parameter $L=1+(\gamma\rho_{\max})^n$ plays an important role in the finite-sample bound. In fact, $L$ appears quadratically in the variance term of Eq. (\ref{eq:Qtrace-bound}), and captures the impact of the cumulative product of the importance sampling factors. To overcome the high variance in off-policy learning (i.e., to make sure that the parameter $L=1+(\gamma\rho_{\max})^n$ does not grow exponentially fast with respect to $n$),
we choose $\lambda\in\mathbb{R}^{|\mathcal{S}|}$ such that  $\rho_{\max}=\max_{s}\lambda(s)(\max_{a}\pi(a|s)/\pi_b(a|s)-1)+1\leq 1/\gamma  $. However, as long as $\lambda\neq \bm{1}$, the limit point of the $\lambda$-averaged $Q$-trace algorithm involves an additional bias term (i.e., the second term on the RHS of Eq. (\ref{eq:Qtrace_performance})) that does not vanish even in the tabular setting. Suppose that the target policy $\pi$ is close to the behavior policy $\pi_b$ in the sense that $\max_{s,a}\pi(a|s)/\pi_b(a|s)$ is not too large ($\max_{s,a}\pi(a|s)/\pi_b(a|s)\leq 1/\gamma$ to be precise), we can still choose $\lambda$ close to $\bm{1}$ to avoid the high variance while also produce a relatively accurate estimate to $Q^\pi$. If that is not the case, then there is a trade-off between the variance in the algorithm and the bias in the limit point in choosing the parameter $\lambda$. Specifically, large $\lambda$ leads to large $\rho_{\max}$ and hence large $L$ and large variance, but in this case the second term on the RHS of Eq. (\ref{eq:Q_performance}) is smaller, implying that we have a smaller bias in the limit point.

The previous discussion also justifies using soft versions of argmax for the actor in updating the policies (as presented in Section \ref{sec:actor}). Recall that in the on-policy setting, API with hardmax policy update suffers from exploration issues as deterministic policies have no exploration components. In the off-policy setting, while in principle hardmax policy update can be used as the behavior policy has taken care of the exploration, it is still not a preferable choice. The reason is that there will be either a large variance or a large bias because  $\pi_t$ and $\pi_b$ are in some sense most different because $\pi_t$ is a deterministic policy but $\pi_b$ should guarantee sufficient exploration for all actions.

Next, we present the finite-sample bounds of the two-sided $Q$-trace algorithm.

\begin{theorem}\label{co:bound-twosided}
	Consider $\{w_k\}$ of the two-sided $Q$-trace Algorithm. Suppose that (1) Assumption \ref{as:MC} is satisfied, (2) the upper and lower truncation levels $\ell,u\in\mathbb{R}^{|\mathcal{S}|}$ are chosen such that $\sum_{a\in\mathcal{A}}\pi_b(a|s)g_{\ell(s),u(s)}(\pi(a|s)/\pi_b(a|s))=1$ for all $s$, (3) the parameter $n$ is chosen such that $\gamma_c:=\gamma^n/\sqrt{\mathcal{K}_{SA,\min}}<1$, and (4) the stepsize $\alpha$ is chosen such that $\alpha(t_\alpha+n+1)\leq \frac{(1-\gamma_c)\lambda_{\min}}{130L^2}$. Then, we have for all $k\geq t_\alpha+n+1$ that
	\begin{align}
		\mathbb{E}[\|w_k-w_{c,\rho}^\pi\|_2^2]\leq c_1(1-(1-\gamma_c)\lambda_{\min}\alpha)^{k-(t_\alpha+n+1)}+c_2\frac{\alpha L^2 (t_\alpha+n+1)}{(1-\gamma_c)\lambda_{\min}},\label{eq:twosidedQ-bound}
	\end{align}
	where $c_1=(\|w_0\|_2+\|w_0-w_{c,\rho}^\pi\|_2+1)^2$ and $c_2=130(\|w_{c,\rho}^\pi\|_2+1)^2$. Moreover, we have
	\begin{align}\label{eq:twosidedQ-performance}
		\|Q^\pi-\Phi w_{c,\rho}^\pi\|_{\mathcal{K}_{SA}}\leq\;& \frac{1}{\sqrt{1-\gamma_c^2}}\|Q_{c,\rho}^\pi-\Phi w_{c,\rho}^\pi\|_{\mathcal{K}_{SA}}\nonumber\\
		&+\frac{\gamma \max_{s\in\mathcal{S}}\sum_{a\in\mathcal{A}}(u_{\pi,\pi_b}(s,a)-\ell_{\pi,\pi_b}(s,a))}{(1-\gamma)^2},
	\end{align}
	where $u_{\pi,\pi_b}(s,a)=\max(\pi(a|s)-\pi_b(a|s)u(s),0)$ and $\ell_{\pi,\pi_b}(s,a)=\min(\pi(a|s)-\pi_b(a|s)\ell(s),0)$.
\end{theorem}

The finite-sample bound of the two-sided $Q$-trace algorithm is qualitatively similar to that of the $\lambda$-averaged $Q$-trace algorithm. To overcome the high variance issue in off-policy learning, we choose the upper truncation level such that $\gamma u(s)\leq 1$ for all $s$, which ensures that the parameter $L=1+(\gamma\rho_{\max})^n$ is upper bounded by $1+(\gamma \max_{s}u(s))^n$, hence does not grow exponentially with respect to $n$. Then we choose the lower truncation level accordingly to satisfy requirement (2) stated in Theorem \ref{co:bound-twosided}. However, as long as there exists $s\in\mathcal{S}$ such that $u(s)<\max_{s,a}\pi(a|s)/\pi_b(a|s)$ or $\ell(s)>\min_{s,a}\pi(a|s)/\pi_b(a|s)$, the second term on the RHS of Eq. (\ref{eq:twosidedQ-performance}) is in general non-zero, hence adding an additional bias term to the limit point even in the tabular setting.
As a result, the trade-off between the variance and the bias in the limit point is also present in the two-sided $Q$-trace algorithm. 

In view of Theorems \ref{co:Qtrace-bound} and \ref{co:bound-twosided}, one limitation of this work is that the choice of $n$ to make $\gamma_c<1$ (which is needed for the stability of Algorithm \ref{algorithm:Off-Policy-TD}) depends on the unknown parameter $\mathcal{K}_{SA,\min}$ of the problem. In practice, one can start with a specific choice of $n$ and then gradually tune $n$ to achieve the convergence of the $\lambda$-averaged $Q$-trace algorithm or the two-sided $Q$-trace algorithm.

The proof of Theorems \ref{co:Qtrace-bound} and \ref{co:bound-twosided} are presented in Section \ref{subec:pf:critic}. In fact, we will prove a more general theorem for Algorithm \ref{algorithm:Off-Policy-TD} with generalized importance sampling factors satisfying Condition \ref{as:IS_ratio}, which covers $\lambda$-averaged $Q$-trace and two-sided $Q$-trace as its special cases. In Section \ref{subec:pf:critic}, we will also present the finite-sample bound for using linearly diminishing stepsizes, as opposed to using only constant stepsize in Theorems \ref{co:Qtrace-bound} and \ref{co:bound-twosided}.

\section{The Sample Complexity of Policy-Based Methods}\label{sec:API}

To this end, we have established the finite-time bound of the actor for using various policy update rules and finite-sample bound of the critic under off-policy sampling and linear function approximation. In this section, we combine the actor and the critic together and present the overall sample complexity bounds of policy-based methods.
\begin{algorithm}[t]\caption{A Generic Policy-Based Algorithm: Off-Policy Sampling and Linear Function Approximation}\label{algorithm:API_real}
	\begin{algorithmic}[1] 
		\STATE {\bfseries Input:} Integers $T,K$, initialization $\pi_0(\cdot|s)\sim \text{Unif}(\mathcal{A})$ for all $s\in\mathcal{S}$, $\theta_0=\bm{0}$, and a sample trajectory $\{(S_t,A_t)\}_{0\leq t\leq T(K+n)}$ collected under the behavior policy $\pi_b$.
		\FOR{$t=0,1,\dots,T-1$}
		\STATE $w_{t+1}=\text{ALG}(\bm{0},\pi_t,\pi_b,\alpha,K,\{(S_k,A_k)\}_{t(K+n)\leq k\leq(t+1)(K+n)-1})$
		\STATE \textit{\textbf{Option I: Natural Policy Gradient:}}\\
		$\theta_{t+1}=\theta_t+\beta_tw_{t+1}$ \\
		Compute $\pi_{t+1}(a|s)=\frac{\exp(\phi(s,a)^\top \theta_{t+1})}{\sum_{a'\in\mathcal{A}}\exp(\phi(s,a')^\top \theta_{t+1})}$
		at $(s,a)=(S_i,A_i)$ for all $i=(t+1)(K+n),\cdots,(t+2)(K+n)-1$.
		\STATE \textit{\textbf{Option II: Boltzmann Softmax Update:}}\\
		Compute $\pi_{t+1}(a|s)=\frac{\exp(\beta_t\phi(s,a)^\top w_{t+1})}{\sum_{a'\in\mathcal{A}}\exp(\beta_t\phi(s,a')^\top w_{t+1})}$
		at $(s,a)=(S_i,A_i)$ for all $i=(t+1)(K+n),\cdots,(t+2)(K+n)-1$.
		\STATE \textit{\textbf{Option III: $\epsilon$-Greedy Update:}}\\
		Compute 
		\begin{align*}
			\pi_{t+1}(a|s)=\begin{dcases}
				\beta_t/|\mathcal{A}|,&a\neq {\arg\max}_{a'\in\mathcal{A}}\phi(s,a')^\top w_{t+1},\\
				\beta_t/|\mathcal{A}|+1-\beta_t,&a= {\arg\max}_{a'\in\mathcal{A}}\phi(s,a')^\top w_{t+1},
			\end{dcases}
		\end{align*}
		at $(s,a)=(S_i,A_i)$ for all $i=(t+1)(K+n),\cdots,(t+2)(K+n)-1$.
		\ENDFOR
		\STATE\textbf{Output:} $\pi_T$
	\end{algorithmic}
\end{algorithm}

We begin by presenting the algorithm design for general policy-based methods in Algorithm \ref{algorithm:API_real}, which can be viewed as an extension of Algorithm \ref{algorithm:API} to the case where off-policy sampling and linear function approximation are used. Note that Algorithm \ref{algorithm:API_real} line 3 corresponds to Algorithm \ref{algorithm:API} line 3, where the critic implements Algorithm \ref{algorithm:Off-Policy-TD} for policy evaluation with off-policy data. For simplicity of notation, for a given target policy $\pi$, behavior policy $\pi_b$, constant stepsize $\alpha$, initialization $w_0$, and samples $\{(S_k,A_k)\}_{0\leq k\leq K+n-1}$, we have denoted the output of Algorithm \ref{algorithm:Off-Policy-TD} after $K$ iterations by $w=\text{ALG}(w_0,\pi,\pi_b,\alpha,K,\{(S_k,A_k)\}_{0\leq k\leq K+n-1})$.

Algorithm \ref{algorithm:API_real} lines 4, 5, 6 correspond to Algorithm \ref{algorithm:API} line 4, where we update the policy. Note that ideally we would like to perform $\pi_{t+1}=G(\Phi w_t,\pi_t)$ for policy improvement, where $\Phi w_t$ is the $Q$-function estimate associated with the weight vector $w_t$, and $G(\cdot,\cdot)$ is the policy update rule. However, directly implementing $\pi_{t+1}=G(\Phi w_t,\pi_t)$ violates the purpose of using function approximation as it requires to compute $\Phi w_t$, $\pi_t$, and the the output of $G(\cdot,\cdot)$, which are all $|\mathcal{S}||\mathcal{A}|$-dimensional objects. Algorithm \ref{algorithm:API_real} lines 4,5,6 can be viewed as an equivalent way of implementing Algorithm \ref{algorithm:API} line 4 without working with $|\mathcal{S}||\mathcal{A}|$-dimensional objects. We next elaborate on such equivalence.

Let us first look at the case of using natural policy gradient. Recall in Eq. (\ref{eq:natural policy gradient_policy}) that natural policy gradient updates the policy according to
\begin{align*}
	\pi_{t+1}(a|s)=\frac{\pi_t(a|s)\exp(\beta_t Q_t(s,a))}{\sum_{a'\in\mathcal{A}}\pi_t(a'|s)\exp(\beta_t Q_t(s,a'))},\quad \forall \;(s,a),
\end{align*}
where $\pi_t$ is the previous policy and $Q_t$ is $\pi_t$'s $Q$-function estimate. Therefore, to use function approximation, parametrizing only the $Q$-function as in Section \ref{sec:PE} is not sufficient, we also need to keep track of the policies. Inspired by \cite{agarwal2019theory}, we use compatible linear function approximation to parametrize the policy, i.e., we approximate policies using $\pi_\theta(a|s)=\frac{\exp(\phi(s,a)^\top \theta)}{\sum_{a'\in\mathcal{A}}\exp(\phi(s,a')^\top \theta)}$, where $\phi(s,a)$ is the same feature we used in approximating the $Q$-functions, and $\theta\in\mathbb{R}^d$ is the weight vector. It has been shown in \cite[Lemma 3.1]{chen2021NACLFA} that with the weight vector being updated according to $\theta_{t+1}=\theta_t+\beta_tw_t$ (which involves only $d$-dimesional objects), the corresponding policy is equivalently updated according to
\begin{align*}
	\pi_{t+1}(a|s)=\frac{\pi_t(a|s)\exp(\beta_t \phi(s,a)^\top w_t)}{\sum_{a'\in\mathcal{A}}\pi_t(a'|s)\exp(\beta_t \phi(s,a')^\top w_t)},\quad \forall \;(s,a),
\end{align*}
which is exactly the natural policy gradient update. 

Next we consider API with either Boltzmann softmax update or $\epsilon$-greedy update. Recall that in these two cases the next policy $\pi_{t+1}$ depends only on the estimate of the $Q$-function $Q_t$, and not on the previous policy. Since we already parametrized the $Q$-function, there is no need to parametrize the policy. Moreover, since implementing the policy evaluation algorithm (cf. Algorithm \ref{algorithm:Off-Policy-TD}) requires only the target policy value at state-action pairs that are visited by the sample trajectory generated by the behavior policy, we do not need to work with $|\mathcal{S}||\mathcal{A}|$-dimensional objects in either the actor or the critic in Algorithm \ref{algorithm:API_real}.

We next present the finite-sample bounds of Algorithm \ref{algorithm:API_real}. For ease of exposition, in Algorithm \ref{algorithm:Off-Policy-TD} for policy evaluation, we consider using the generalized importance sampling factors according to $\lambda$-averaged $Q$-trace. For two sided $Q$-trace or more general importance sampling factors (as long as Condition \ref{as:IS_ratio} is satisfied), the results are straight forward extensions.

\begin{theorem}\label{thm:combine}
	Consider $\pi_T$ generated by Algorithm \ref{algorithm:API_real}. Suppose that the assumptions for applying Theorem \ref{co:Qtrace-bound} are satisfied. Then under Condition \ref{condition:stepsize}, we have
	\begin{align}
		\mathbb{E}[\|Q^*-Q^{\pi_T}\|_\infty]
		\leq\;& \underbrace{\gamma^T\|Q^*-Q^{\pi_0}\|_\infty}_{N_1:\text{ Convergence bias in the actor}}+\underbrace{\frac{2\gamma\mathcal{E}_{\text{approx}}}{(1-\gamma)^2}}_{N_{2,1}}+\underbrace{\frac{2\gamma^2\mathcal{E}_{\text{bias}}}{(1-\gamma)^4}}_{N_{2,2}}+\nonumber\\
		&+\underbrace{\frac{6(1-(1-\gamma_c)\lambda_{\min}\alpha)^{\frac{1}{2}[K-(t_\alpha+n+1)]}}{(1-\gamma)^3(1-\gamma_c)^{1/2}\lambda_{\min}^{1/2}}}_{N_{2,3}: \text{ Convergence bias in the critic}}+\underbrace{\frac{70L[\alpha(t_\alpha+n+1)]^{1/2}}{\lambda_{\min}(1-\gamma_c)(1-\gamma)^3}}_{N_{2,4}: \text{ Critic variance}}\nonumber\\
		&+\underbrace{\frac{2\gamma \beta}{(1-\gamma)^2}}_{N_3},\label{eq:100}
	\end{align}
	where $\mathcal{E}_{\text{approx}}=\sup_{\pi}\|Q_{c,\rho}^\pi-\Phi w_{c,\rho}^\pi\|_\infty$ and
	$\mathcal{E}_{\text{bias}}=\max_{0\leq t\leq T}\max_{s\in\mathcal{S}}(1-\lambda(s))\|\pi_t(\cdot|s)-\pi_b(\cdot|s)\|_1$. Under Condition \ref{condition:increasing_stepsize}, we have
	\begin{align*}
		\mathbb{E}[\|Q^*-Q^{\pi_T}\|_\infty]\leq N_1+\sum_{i=1}^4N_{2,i}+\underbrace{\frac{2\gamma^T}{(1-\gamma)^2}}_{N_3'}.
	\end{align*}
\end{theorem}

Notably on the LHS, our finite-sample guarantees are stated for the last policy iterate $\pi_T$, while in many existing literature it was stated for the best policy among $\{\pi_t\}_{0\leq t\leq T}$ \citep{agarwal2019theory}, or a uniform sample from $\{\pi_t\}_{0\leq t\leq T-1}$ \citep{chen2021NAC,xu2020improving}.

The terms $N_1$ and $N_3$ (or $N_3'$) are the same terms appeared in Theorem \ref{thm:main}, and together capture the error in the actor update. This was illustrated right after Theorem \ref{thm:main}. The terms $\{N_{2,i}\}_{1\leq i\leq 4}$ correspond to the term $N_2$ in Theorem \ref{thm:main}, and represent the error in the critic. Specifically, $N_{2,1}$ is the function approximation error, and is equal to zero when we use a complete basis (i.e., tabular setting). A term of similar form is present in all existing work studying RL with function approximation, see for example the inherent Bellman error in \cite{munos2003error}, and the transfer error in \cite{agarwal2019theory}. The term $N_{2,2}$ represents the bias introduced to the algorithm by using generalized importance sampling factors $c(\cdot,\cdot)$ and $\rho(\cdot,\cdot)$. Note that we have $N_{2,2}=0$ when $c(s,a)=\rho(s,a)=\pi(a|s)/\pi_b(a|s)$, which corresponds to using $\lambda=\bm{1}$ in the $\lambda$-averaged $Q$-trace algorithm, and using $u(s)\geq \max_{s,a}\pi(a|s)/\pi_b(a|s)$ and $\ell(s)\leq \min_{s,a}\pi(a|s)/\pi_b(a|s)$ for all $s$ in the two-sided $Q$-trace algorithm. However, this choice of $\lambda$ (or $u$ and $\ell$) might lead to a high variance. In particular, the parameter $L$ within the term $N_{2,4}$ could be large. Such bias-variance trade-off was illustrated in Section \ref{sec:PE}. The term $N_{2,3}$ represents the convergence bias in the critic, and goes to zero geometrically fast as the inner loop iteration index $K$ goes to infinity. The term $N_{2,4}$ represents the variance in the critic, and is proportional to $\sqrt{\alpha t_\alpha}=\mathcal{O}(\sqrt{\alpha\log(1/\alpha)})$. Therefore, $N_{2,4}$ can be made arbitrarily small by using small enough stepsize $\alpha$.

Based on Theorem \ref{thm:combine}, we next derive the sample complexity of Algorithm \ref{algorithm:API_real}. To enable fair comparison with existing literature, we choose $\lambda=\bm{1}$ to eliminate the error due to using generalized importance sampling factors. Note that $\lambda=\bm{1}$ implies $\mathcal{E}_{\text{bias}}=0$ (and hence $N_{2,2}=0$) in Theorem \ref{thm:combine}.

\begin{corollary}
	For a given accuracy level $\epsilon>0$, to achieve $\mathbb{E}[\|Q^*-Q^{\pi_T}\|_\infty]\leq \epsilon+N_{2,1}$, the number of samples (i.e., the integer $TK$) required is of the size $\mathcal{O}\left(\frac{\log^3(1/\epsilon)}{\epsilon^2}\right)\Tilde{\mathcal{O}}\left(\frac{L^2n}{(1-\gamma)^7(1-\gamma_c)^3\lambda_{\min}^3}\right)$.
\end{corollary}

Notably, we obtain $\Tilde{\mathcal{O}}(\epsilon^{-2})$ sample  complexity for policy-based methods, which matches with the sample complexity of value-based algorithms such as $Q$-learning \citep{li2020sample}. In the case of natural actor-critic, to our knowledge, \cite{cayci2021linear,lan2021policy,cen2021fast} establish the $\tilde{\mathcal{O}}(\epsilon^{-2})$ sample complexity of on-policy natural actor-critic under regularization, and \cite{chen2021NACLFA} establishes the $\tilde{\mathcal{O}}(\epsilon^{-3})$ sample complexity of a variant of off-policy natural actor-critic. We improve the sample complexity in \cite{chen2021NACLFA} by a factor of $\epsilon^{-1}$, and we do not use any regularization. 

In addition to the dependence on $\epsilon$,  the dependence on $1/(1-\gamma)$ (which is usually called the effective horizon) is also improved by a factor of $1/(1-\gamma)$ compared to existing work \citep{agarwal2019theory,chen2021NACLFA}. The bootstrapping parameter $n$ appears linearly in our sample complexity bound. This matches with the results for $n$-step TD-learning in the on-policy tabular setting \citep{chen2021finite}.

The proof of Theorem \ref{thm:combine} follows from combining the convergence rate analysis of the actor in Section \ref{sec:actor} and the critic in Section \ref{sec:PE}, and is presented in Section \ref{subsec:pf:main}.

\section{Proof of the Main Results}
In this section, we present the proof of Theorems \ref{thm:main}, \ref{co:Qtrace-bound}, \ref{co:bound-twosided}, and  \ref{thm:combine}.

\subsection{Proof of Theorem \ref{thm:main}}\label{subsec:pf:actor}

Following the proof sketch presented in Section \ref{subsec:theorem1sketch}, we next carry out the details.

\subsubsection{Step One.}
For simplicity of notation, denote $\delta_t=\max_{s,a}(Q^{\pi_t}(s,a)-Q^{\pi_{t+1}}(s,a))$ and $\zeta_t=\max_{s,a}(Q^*(s,a)-Q^{\pi_t}(s,a))=\|Q^*-Q^{\pi_t}\|_\infty$ for all $t=0,1,\cdots,T$. Note that we have by definition that $Q^{\pi_{t+1}}\geq Q^{\pi_t}-\delta_t\bm{1}$. Using the monotonicity of the Bellman operator \cite[Lemma 2.1 and Lemma 2.2]{bertsekas1996neuro}  and we have
\begin{align*}
	Q^{\pi_{t+1}}=\mathcal{H}_{\pi_{t+1}}(Q^{\pi_{t+1}})\geq \mathcal{H}_{\pi_{t+1}}(Q^{\pi_t}-\delta_t\bm{1})=\mathcal{H}_{\pi_{t+1}}(Q^{\pi_t})-\gamma\delta_t\bm{1}.
\end{align*}
It follows that
\begin{align*}
	&Q^{\pi_t}-Q^{\pi_{t+1}}\\
	\leq \;& Q^{\pi_t}-\mathcal{H}_{\pi_{t+1}}(Q^{\pi_t})+\gamma\delta_t\bm{1}\\
	=\;&Q^{\pi_t}-\mathcal{H}_{\pi_{t+1}}(Q^{\pi_t})+\mathcal{H}_{\pi_{t+1}}(Q_t)-\mathcal{H}_{\pi_{t+1}}(Q_t)+\mathcal{H}(Q_t)-\mathcal{H}(Q_t)+\gamma\delta_t\bm{1}\\
	\leq \;& \mathcal{H}_{\pi_t}(Q^{\pi_t})-\mathcal{H}_{\pi_t}(Q_t)-\mathcal{H}_{\pi_{t+1}}(Q^{\pi_t})+\mathcal{H}_{\pi_{t+1}}(Q_t)-\mathcal{H}_{\pi_{t+1}}(Q_t)+\mathcal{H}(Q_t)+\gamma\delta_t\bm{1}\\
	\leq  \;&2\gamma \|Q^{\pi_t}-Q_t\|_\infty\bm{1}+\|\mathcal{H}_{\pi_{t+1}}(Q_t)-\mathcal{H}(Q_t)\|_\infty\bm{1}+\gamma\delta_t\bm{1}.
\end{align*}
Therefore, we have $\delta_t\leq 2\gamma \|Q^{\pi_t}-Q_t\|_\infty+\|\mathcal{H}_{\pi_{t+1}}(Q_t)-\mathcal{H}(Q_t)\|_\infty+\gamma\delta_t$,
which implies
\begin{align}\label{eqeq:1}
	\delta_t\leq \frac{2\gamma \|Q^{\pi_t}-Q_t\|_\infty+\|\mathcal{H}_{\pi_{t+1}}(Q_t)-\mathcal{H}(Q_t)\|_\infty}{1-\gamma}.
\end{align}
Using the monotonicity of the Bellman operator and we have
\begin{align}
	Q^{\pi_{t+1}}&=\mathcal{H}_{\pi_{t+1}}(Q^{\pi_{t+1}})\nonumber\\
	&\geq \mathcal{H}_{\pi_{t+1}}(Q^{\pi_t}-\max_{s,a}(Q^{\pi_t}(s,a)-Q^{\pi_{t+1}}(s,a))\bm{1})\nonumber\\
	&=\mathcal{H}_{\pi_{t+1}}(Q^{\pi_t})-\gamma \max_{s,a}(Q^{\pi_t}(s,a)-Q^{\pi_{t+1}}(s,a))\bm{1}\nonumber\\
	&\geq \mathcal{H}_{\pi_{t+1}}(Q^{\pi_t})- \frac{2\gamma^2 \|Q^{\pi_t}-Q_t\|_\infty+\gamma\|\mathcal{H}_{\pi_{t+1}}(Q_t)-\mathcal{H}(Q_t)\|_\infty}{1-\gamma}\bm{1},\label{eq:1}
\end{align}
where the last line follows from Eq. (\ref{eqeq:1}). We next control $\mathcal{H}_{\pi_{t+1}}(Q^{\pi_t})$ from below in the following. Again by monotonicity of the Bellman operator we have
\begin{align*}
	\mathcal{H}_{\pi_{t+1}}(Q^{\pi_t})&\geq \mathcal{H}_{\pi_{t+1}}(Q_t-\|Q_t-Q^{\pi_t}\|_\infty\bm{1})\\
	&=\mathcal{H}_{\pi_{t+1}}(Q_t)-\gamma\|Q_t-Q^{\pi_t}\|_\infty\bm{1}\\
	&=\mathcal{H}_{\pi_{t+1}}(Q_t)-\mathcal{H}(Q_t)+\mathcal{H}(Q_t)-\gamma\|Q_t-Q^{\pi_t}\|_\infty\bm{1}\\
	&\geq \mathcal{H}_{\pi_{t+1}}(Q_t)-\mathcal{H}(Q_t)+\mathcal{H}(Q^{\pi_t}-\|Q_t-Q^{\pi_t}\|_\infty\bm{1})-\gamma\|Q_t-Q^{\pi_t}\|_\infty\bm{1}\\
	&=\mathcal{H}_{\pi_{t+1}}(Q_t)-\mathcal{H}(Q_t)+\mathcal{H}(Q^{\pi_t})-2\gamma\|Q_t-Q^{\pi_t}\|_\infty\bm{1}\\
	&\geq \mathcal{H}_{\pi_{t+1}}(Q_t)-\mathcal{H}(Q_t)+\mathcal{H}(Q^*-\zeta_t\bm{1})-2\gamma\|Q_t-Q^{\pi_t}\|_\infty\bm{1}\\
	&=\mathcal{H}_{\pi_{t+1}}(Q_t)-\mathcal{H}(Q_t)+\mathcal{H}(Q^*)-\gamma\zeta_t\bm{1}-2\gamma\|Q_t-Q^{\pi_t}\|_\infty\bm{1}\\
	&\geq -\|\mathcal{H}_{\pi_{t+1}}(Q_t)-\mathcal{H}(Q_t)\|_\infty\bm{1}+Q^*-\gamma\zeta_t\bm{1}-2\gamma\|Q_t-Q^{\pi_t}\|_\infty\bm{1}.
\end{align*}
Using the previous inequality in Eq. (\ref{eq:1}) and we have
\begin{align*}
	Q^{\pi_{t+1}}-Q^*
	&\geq -\gamma\zeta_t\bm{1}- \frac{2\gamma \|Q^{\pi_t}-Q_t\|_\infty+\|\mathcal{H}_{\pi_{t+1}}(Q_t)-\mathcal{H}(Q_t)\|_\infty}{1-\gamma}\bm{1},
\end{align*}
which implies
\begin{align*}
	\zeta_{t+1}\leq \gamma\zeta_t+ \frac{2\gamma \|Q^{\pi_t}-Q_t\|_\infty+\|\mathcal{H}_{\pi_{t+1}}(Q_t)-\mathcal{H}(Q_t)\|_\infty}{1-\gamma}.
\end{align*}
Repeatedly using the previous inequality and we obtain
\begin{align}\label{eq:actor_recursion}
	\zeta_T\leq\;& \gamma^T\zeta_0+\frac{2\gamma}{1-\gamma}\sum_{t=0}^{T-1}\gamma^{T-1-t}\|Q^{\pi_t}-Q_t\|_\infty
	\nonumber\\
	&+\frac{2\gamma}{1-\gamma}\sum_{t=0}^{T-1}\gamma^{T-1-t}\|\mathcal{H}_{\pi_{t+1}}(Q_t)-\mathcal{H}(Q_t)\|_\infty.
\end{align}

\subsubsection{Step Two.}

We now bound the term $\|\mathcal{H}_{\pi_{t+1}}(Q_t)-\mathcal{H}(Q_t)\|_\infty$, $t=0,\cdots,T-1$, in Eq. (\ref{eq:actor_recursion}). We first consider using natural policy gradient in Algorithm \ref{algorithm:API} line 4.

\paragraph{Natural Policy Gradient.}
For any $0\leq t\leq T-1$ and state-action pair $(s,a)$, using the definition of the Bellman operators $\mathcal{H}(\cdot)$ and $\mathcal{H}_\pi(\cdot)$ and we have
\begin{align}
	0&\leq [\mathcal{H}(Q_t)](s,a)-[\mathcal{H}_{\pi_{t+1}}(Q_t)](s,a)\nonumber\\
	&=\gamma \sum_{s'}P_a(s,s')\left(\max_{a'\in\mathcal{A}}Q_t(s',a')-\sum_{a'\in\mathcal{A}}\frac{\pi_t(a'|s')\exp(\beta_t Q_t(s',a'))}{\sum_{a''\in\mathcal{A}}\pi_t(a''|s')\exp(\beta_t Q_t(s',a'')}Q_t(s',a')\right).\label{eq:NPG10}
\end{align}

To proceed and bound the RHS of the previous inequality, observe that that for any policy $\pi$, $Q\in\mathbb{R}^{|\mathcal{S}||\mathcal{A}|}$, and $(s,a)$, we have
\begin{align*}
	\lim_{\beta \rightarrow\infty}\frac{\pi(a|s)\exp{(\beta Q(s,a))}}{\sum_{a'\in\mathcal{A}}\pi(a'|s)\exp(\beta Q(s,a'))}=\;&\lim_{\beta \rightarrow\infty}\frac{\pi(a|s)\exp{(\beta (Q(s,a)-Q_{s,\max}))}}{\sum_{a'\in\mathcal{A}}\pi(a'|s)\exp(\beta (Q(s,a')-Q_{s,\max}))}\\
	=\;&\begin{dcases}
		\frac{1}{|\{a\in\mathcal{A}\mid  Q(s,a)=Q_{s,\max}\}|},&Q(s,a)=Q_{s,\max},\\
		0,&Q(s,a)<Q_{s,\max},
	\end{dcases}
\end{align*}
where $Q_{s,\max}=\max_{a'\in\mathcal{A}}Q(s,a)$, and hence
\begin{align*}
	\lim_{\beta\rightarrow\infty}\sum_{a'\in\mathcal{A}}\frac{\pi_t(a'|s')\exp(\beta Q_t(s',a'))}{\sum_{a''\in\mathcal{A}}\pi_t(a''|s')\exp(\beta Q_t(s',a'')}Q_t(s',a')=\max_{a'\in\mathcal{A}}Q_t(s',a'),\quad\forall\;s',t.
\end{align*}
Therefore, natural policy gradient is an approximate version of PI when $\beta$ is large. To explicitly bound the difference between natural policy gradient and PI for using a finite stepsize, the following lemma is needed, the proof of which is presented in Appendix \ref{pf:le:difference}.

\begin{lemma}\label{le:difference}
	Let $x\in\mathbb{R}^d$ be arbitrary and $y\in\Delta^d$ satisfying $y_i>0$ for all $i$. Denote $i_{\max}=\arg\max_{1\leq i\leq d}x_i$. Then we have $\max_{1\leq i\leq d}x_i-\frac{\sum_{i=1}^dx_iy_ie^{\beta x_i}}{\sum_{j=1}^dy_je^{\beta x_j}}\leq \frac{1}{\beta}\log\left(\frac{1}{y_{i_{\max}}}\right)$ for any $\beta>0$.
\end{lemma}

Using Lemma \ref{le:difference} to bound the RHS of Eq. (\ref{eq:NPG10}) and we have
\begin{align*}
	0&\leq [\mathcal{H}(Q_t)](s,a)-[\mathcal{H}_{\pi_{t+1}}(Q_t)](s,a)\\
	&=\gamma \sum_{s'}P_a(s,s')\left(\max_{a'\in\mathcal{A}}Q_t(s',a')-\sum_{a'\in\mathcal{A}}\frac{\pi_t(a'|s')\exp(\beta_t Q_t(s',a'))}{\sum_{a''\in\mathcal{A}}\pi_t(a''|s')\exp(\beta_t Q_t(s',a'')}Q_t(s',a')\right)\\
	&\leq \gamma \sum_{s'}P_a(s,s')\left(\frac{1}{\beta_t}\log\left(\frac{1}{\pi_t(a_{t,s'}\mid s')}\right)\right)\\
	&\leq \begin{dcases}
		\beta,&\text{Under Condition \ref{condition:stepsize}},\\
		\gamma^{2t},&\text{Under Condition \ref{condition:increasing_stepsize}}.
	\end{dcases}
\end{align*}
Therefore, under Condition \ref{condition:stepsize}, we have from Eq. (\ref{eq:actor_recursion}) that
\begin{align*}
	\|Q^*-Q^{\pi_T}\|_\infty\leq\;& \gamma^T\|Q^*-Q^{\pi_0}\|_\infty\!+\!\frac{2\gamma}{1-\gamma}\sum_{t=0}^{T-1}\gamma^{T-1-t}\|Q^{\pi_t}-Q_t\|_\infty
	\!+\!\frac{2\gamma}{1-\gamma}\sum_{t=0}^{T-1}\gamma^{T-1-t}\beta\\
	= \;&\gamma^T\|Q^*-Q^{\pi_0}\|_\infty\!+\!\frac{2\gamma}{1-\gamma}\sum_{t=0}^{T-1}\gamma^{T-1-t}\|Q^{\pi_t}-Q_t\|_\infty
	\!+\!\frac{2\gamma \beta}{(1-\gamma)^2}.
\end{align*}
Under Condition \ref{condition:increasing_stepsize}, we have from Eq. (\ref{eq:actor_recursion}) that
\begin{align*}
	\|Q^*-Q^{\pi_T}\|_\infty\leq\;& \gamma^T\|Q^*-Q^{\pi_0}\|_\infty+\frac{2\gamma}{1-\gamma}\sum_{t=0}^{T-1}\gamma^{T-1-t}\|Q^{\pi_t}-Q_t\|_\infty
	+\frac{2\gamma}{1-\gamma}\sum_{t=0}^{T-1}\gamma^{T-1+t}\\
	= \;&\gamma^T\|Q^*-Q^{\pi_0}\|_\infty+\frac{2\gamma}{1-\gamma}\sum_{t=0}^{T-1}\gamma^{T-1-t}\|Q^{\pi_t}-Q_t\|_\infty
	+\frac{2\gamma^T}{(1-\gamma)^2}.
\end{align*}

\paragraph{Boltzmann Softmax Update.} We next consider using Boltzmann softmax update for Algorithm \ref{algorithm:API} line 4. For any $0\leq t\leq T-1$ and state-action pair $(s,a)$, we have from the update rule that
\begin{align*}
	0&\leq [\mathcal{H}(Q_t)](s,a)-[\mathcal{H}_{\pi_{t+1}}(Q_t)](s,a)\\
	&=\gamma \sum_{s'}P_a(s,s')\left(\max_{a'\in\mathcal{A}}Q_t(s',a')-\sum_{a'\in\mathcal{A}}\frac{\exp(\beta_t Q_t(s',a'))}{\sum_{a''\in\mathcal{A}}\exp(\beta_t Q_t(s',a'')}Q_t(s',a')\right)\\
	&=\gamma \sum_{s'}P_a(s,s')\left(\max_{a'\in\mathcal{A}}Q_t(s',a')-\sum_{a'\in\mathcal{A}}\frac{\exp(\beta_t Q_t(s',a'))/|\mathcal{A}|}{\sum_{a''\in\mathcal{A}}\exp(\beta_t Q_t(s',a'')/|\mathcal{A}|}Q_t(s',a')\right)\\
	&\leq \frac{\gamma}{\beta_t} \log(|\mathcal{A}|)\tag{Lemma \ref{le:difference}}\\
	&\leq \begin{dcases}
		\beta,&\text{Under Condition \ref{condition:stepsize}},\\
		\gamma^{2t},&\text{Under Condition \ref{condition:increasing_stepsize}},\\
	\end{dcases}
\end{align*}
The result then follows from using the previous inequality in Eq. (\ref{eq:actor_recursion}).

\paragraph{$\epsilon$-Greedy Update.} Finally, we consider using $\epsilon$-greedy update in Algorithm \ref{algorithm:API} line 4. For any $0\leq t\leq T-1$ and state-action pair $(s,a)$ that
\begin{align*}
	0&\leq [\mathcal{H}(Q_t)](s,a)-[\mathcal{H}_{\pi_{t+1}}(Q_t)](s,a)\\
	&=\gamma\mathbb{E}\left[\max_{a'\in\mathcal{A}}Q_t(S_1,a')-Q_t(S_1,A_1)\;\middle|\; S_0=s,A_0=a\right]\tag{$A_1\sim \pi_{t+1}(\cdot|S_1)$}\\
	&=\gamma \sum_{s'}P_a(s,s')\left(\left(\frac{1}{\beta(t,s')}-\frac{1}{|\mathcal{A}|\beta(t,s')}\right)Q_t(s',a_{t,s'})-\sum_{a'\neq a_{t,s'}}\frac{1}{|\mathcal{A}|\beta(t,s')}Q_t(s',a')\right)\\
	&\leq \gamma \sum_{s'}P_a(s,s') \frac{2}{\beta(t,s')}\max_{a'\in\mathcal{A}}|Q_t(s',a')|\\
	& \leq \begin{dcases}
		\beta,&\text{Under Condition \ref{condition:stepsize}},\\
		\gamma^{2t},&\text{Under Condition \ref{condition:increasing_stepsize}}.
	\end{dcases}
\end{align*}
The result follows from using the previous inequality in Eq. (\ref{eq:actor_recursion}). The proof of Theorem \ref{thm:main} is now complete.

\subsection{Proof of Theorem \ref{co:Qtrace-bound} and Theorem \ref{co:bound-twosided}}\label{subec:pf:critic}

Instead of proving Theorems \ref{co:Qtrace-bound} and \ref{co:bound-twosided}, we will state and prove finite-sample bounds for Algorithm \ref{algorithm:Off-Policy-TD} with $c(\cdot,\cdot)$ and $\rho(\cdot,\cdot)$ satisfying Condition \ref{as:IS_ratio}, which subsumes Theorems \ref{co:Qtrace-bound} and \ref{co:bound-twosided} as its special cases. In this more general setup where we do not necessarily have $c(\cdot,\cdot)=\rho(\cdot,\cdot)$, we define $L$ as
\begin{align}\label{def:L}
	L=\begin{dcases}
		(1+(\gamma \rho_{\max})^n),&c(\cdot,\cdot)=\rho(\cdot,\cdot),\\
		(1+\gamma \rho_{\max})f_n(\gamma c_{\max}),&c(\cdot,\cdot)\neq \rho(\cdot,\cdot),
	\end{dcases}
\end{align}
where $c_{\max}=\max_{s,a}c(s,a)$ and $\rho_{\max}=\max_{s,a}\rho(s,a)$. Moreover, in addition to using constant stepsize $\alpha_k\equiv\alpha$, we also consider using diminishing stepsizes of the form $\alpha_k=\frac{\alpha}{k+h}$, where $\alpha,h>0$. In this case, for simplicity of notation, we write $t_k$ for $t_{\alpha_k}$, which is the mixing time of the Markov chain $\{S_k\}$ (induced by $\pi_b$) with precision $\alpha_k$.

\begin{theorem}\label{thm:Off-Policy-TD}
	Consider $\{w_k\}$ of Algorithm \ref{algorithm:Off-Policy-TD}. Suppose that (1) Assumption  \ref{as:MC} is satisfied, (2) the generalized importance sampling factors satisfy Condition \ref{as:IS_ratio}, (3) the parameter $n$ is chosen such that $\gamma_c:=\tilde{\gamma}(n)/\sqrt{\mathcal{K}_{SA,\min}}<1$. Then, when using constant stepsize $\alpha_k\equiv \alpha$, where $\alpha$ is chosen such that $\alpha(t_\alpha+n+1)\leq \frac{(1-\gamma_c)\lambda_{\min}}{130L^2}$, we have for all $k\geq t_\alpha+n+1$:
	\begin{align}\label{eq:bound}
		\mathbb{E}[\|w_k-w_{c,\rho}^\pi\|_2^2]\leq c_1(1-(1-\gamma_c)\lambda_{\min}\alpha)^{k-(t_\alpha+n+1)}+c_2L^2\frac{\alpha (t_\alpha+n+1)}{(1-\gamma_c)\lambda_{\min}},
	\end{align}
	where $c_1=(\|w_0\|_2+\|w_0-w_{c,\rho}^\pi\|_2+1)^2$ and $c_2=130(\|w_{c,\rho}^\pi\|_2+1)^2$. When using diminishing stepsizes of the form $\alpha_k=\frac{\alpha}{k+h}$ with $\alpha>\frac{1}{(1-\gamma_c)\lambda_{\min}}$ and $h$ chosen such that $\sum_{i=k-(t_k+n+1)}^{k-1}\alpha_i\leq \frac{(1-\gamma_c)\lambda_{\min}}{130L^2}$ for all $k\geq t_k+n+1$, we have 
	\begin{align*}
		\mathbb{E}[\|w_k-w_{c,\rho}^\pi\|_2^2]\leq c_1\frac{k_0+h}{k+h}+c_2\frac{8e\alpha^2}{(1-\gamma_c)\lambda_{\min}\alpha-1}\frac{t_k+n+1}{k+h}
	\end{align*}
	for all $k\geq k_0$, where $k_0:=\min\{k\;:\;k\geq t_k+n+1\}$.
\end{theorem}

To prove Theorem \ref{thm:Off-Policy-TD}, we first rewrite Algorithm \ref{algorithm:Off-Policy-TD} as a stochastic approximation algorithm. Let $\{X_k\}$ be a finite-state Markov chain defined as $X_k=(S_k,A_k,...,S_{k+n},A_{k+n})$ for any $k\geq 0$. Denote the state-space of $\{X_k\}$ by $\mathcal{X}$. It is clear that under Assumption \ref{as:MC}, the Markov chain $\{X_k\}$ also admits a unique stationary distribution, which we denote by $\nu\in\Delta^{|\mathcal{X}|}$. Let $F:\mathbb{R}^d\times \mathcal{X}\mapsto\mathbb{R}^d$ be an operator defined as
\begin{align*}
	F(w,x)=\phi(s_0,a_0)\sum_{i=0}^{n-1}\gamma^i c_{1,i}\Delta_i(w)
\end{align*}
for any $w\in\mathbb{R}^d$ and $x=(s_0,a_0,...,s_n,a_n)\in\mathcal{X}$.
Let $\Bar{F}:\mathbb{R}^d\mapsto\mathbb{R}^d$ be the ``expected'' version of $F(\cdot,\cdot)$ defined by $\Bar{F}(w)=\mathbb{E}_{X\sim\nu}[F(w,X)]$. Using the notation above, the update equation in line 4 of Algorithm \ref{algorithm:Off-Policy-TD} can be compactly written as
\begin{align}\label{eq:sa}
	w_{k+1}&=w_k+\alpha_kF(w_k,X_k),
\end{align}
which is a stochastic approximation algorithm for solving the equation $\Bar{F}(w)=0$ with Markovian noise. Note that $\Bar{F}(w)=0$ is equivalent to the generalized PBE (\ref{eq:pbe-key}) (cf. Lemma \ref{le:equation}). We next establish the properties of the operators $F(\cdot,\cdot)$, $\Bar{F}(\cdot)$, and the Markov chain $\{X_k\}$ in the following proposition, which enables us to use standard stochastic approximation results in the literature to derive finite-sample bounds of Algorithm \ref{algorithm:Off-Policy-TD}.

\begin{proposition}\label{prop:PE:properties}
	The following statements hold:
	\begin{enumerate}[(1)]
		\item $\|F(w_1,x)-F(w_2,x)\|_2\leq L\|w_1-w_2\|_2$ for any $w_1,w_2\in\mathbb{R}^d$ and $x\in\mathcal{X}$, and $\|F(\bm{0},x)\|_2\leq f_n(\gamma c_{\max})$ for any $x\in\mathcal{X}$,
		\item $\max_{x\in\mathcal{X}}\left\|P_X^{k+n+1}(x,\cdot)-\nu(\cdot)\right\|_{\text{TV}}\leq C\sigma^k$ for all $k\geq  0$, where $P_X$ is the transition probability matrix of the Markov chain $\{X_k\}$,
		\item $(w-w_{c,\rho}^\pi)^\top\Bar{F}(w)\leq -(1-\gamma_c)\lambda_{\min}\|w-w_{c,\rho}^\pi\|_2^2$ for any $w\in\mathbb{R}^d$.
	\end{enumerate}
\end{proposition}

\begin{proof}[Proof of Proposition \ref{prop:PE:properties}]
\begin{enumerate}[(1)]
	\item 
	We first rewrite the operator $F(\cdot,\cdot)$ in the following equivalent way. For any $w\in\mathbb{R}^d$ and $x=(s_0,a_0,...,s_n,a_n)\in\mathcal{X}$, we have
	\begin{align*}
		&F(w,x)\\
		=\;&\phi(s_0,a_0)\sum_{i=0}^{n-1}\gamma^i\prod_{j=1}^ic(s_j,a_j)(\mathcal{R}(s_i,a_i)\!+\!\gamma \rho(s_{i+1},a_{i+1})\phi(s_{i+1},a_{i+1})^\top w \!-\!\phi(s_i,a_i)^\top w)\\
		=\;&\phi(s_0,a_0)\sum_{i=0}^{n-1}\gamma^i\prod_{j=1}^ic(s_j,a_j)\mathcal{R}(s_i,a_i)-\phi(s_0,a_0)\sum_{i=0}^{n-1}\gamma^i\prod_{j=1}^ic(s_j,a_j)\phi(s_i,a_i)^\top w\\
		&+\phi(s_0,a_0)\sum_{i=0}^{n-1}\gamma^{i+1}\prod_{j=1}^ic(s_j,a_j)\rho(s_{i+1},a_{i+1})\phi(s_{i+1},a_{i+1})^\top w\\
		=\;&\phi(s_0,a_0)\sum_{i=0}^{n-1}\gamma^i\prod_{j=1}^ic(s_j,a_j)\mathcal{R}(s_i,a_i)-\phi(s_0,a_0)\sum_{i=0}^{n-1}\gamma^i\prod_{j=1}^ic(s_j,a_j)\phi(s_i,a_i)^\top w\\
		&+\phi(s_0,a_0)\sum_{i=1}^{n}\gamma^i\prod_{j=1}^{i-1}c(s_j,a_j)\rho(s_i,a_i)\phi(s_i,a_i)^\top w\\
		=\;&\phi(s_0,a_0)\sum_{i=0}^{n-1}\gamma^i\prod_{j=1}^ic(s_j,a_j)\mathcal{R}(s_i,a_i)-\phi(s_0,a_0)\phi(s_0,a_0)^\top w\\
		&+\phi(s_0,a_0)\sum_{i=1}^{n-1}\gamma^i\prod_{j=1}^{i-1}c(s_j,a_j)(\rho(s_i,a_i)-c(s_i,a_i))\phi(s_i,a_i)^\top w\\
		&+\phi(s_0,a_0)\gamma^n\prod_{j=1}^{n-1}c(s_j,a_j)\rho(s_n,a_n)\phi(s_n,a_n)^\top w.
	\end{align*}
	We now proceed and show the Lipschitz property. For any $w_1,w_2\in\mathbb{R}^d$ and $x=(s_0,a_0,...,s_n,a_n)\in\mathcal{X}$, using the fact that $\|\phi(s,a)\|_2\leq \|\phi(s,a)\|_1\leq \|\Phi\|_\infty\leq 1$, we have
	\begin{align*}
		&\|F(w_1,x)-F(w_2,x)\|_2\\
		\leq \;&\|\phi(s_0,a_0)\phi(s_0,a_0)^\top(w_1-w_2)\|_2\\
		&+\left\|\phi(s_0,a_0)\sum_{i=1}^{n-1}\gamma^i\prod_{j=1}^{i-1}c(s_j,a_j)(\rho(s_i,a_i)-c(s_i,a_i))\phi(s_i,a_i)^\top (w_1-w_2)\right\|_2\\
		&+\left\|\phi(s_0,a_0)\gamma^n\prod_{j=1}^{n-1}c(s_j,a_j)\rho(s_n,a_n)\phi(s_n,a_n)^\top ( w_1-w_2)\right\|_2\\
		\leq \;&\|w_1-w_2\|_2+\sum_{i=1}^{n-1}\gamma^ic_{\max}^{i-1}\max_{s,a}|\rho(s,a)-c(s,a)|\|w_1-w_2\|_2\\
		&+\gamma^nc_{\max}^{n-1}\rho_{\max}\|w_1-w_2\|_2\\
		=\;&\left(1+\gamma \max_{s,a}|\rho(s,a)-c(s,a)|\frac{1-(\gamma c_{\max})^{n-1}}{1-\gamma c_{\max}}+\gamma^n c_{\max}^{n-1}\rho_{\max}\right)\|w_1-w_2\|_2\\
		\leq \;&\begin{dcases}
			(1+(\gamma \rho_{\max})^n)\|w_1-w_2\|_2,&c(\cdot,\cdot)=\rho(\cdot,\cdot)\\
			(1+\gamma \rho_{\max})f_n(\gamma c_{\max})\|w_1-w_2\|_2,&c(\cdot,\cdot)\neq \rho(\cdot,\cdot).
		\end{dcases}
	\end{align*}
	This proves that
	\begin{align*}
		\|F(w_1,x)-F(w_2,x)\|_2\leq L\|w_1-w_2\|_2,\quad\forall\;w_1,w_2\in\mathbb{R}^d.
	\end{align*}
	Similarly, for any $x=(s_0,a_0,...,s_n,a_n)\in\mathcal{X}$, we have
	\begin{align*}
		\|F(\bm{0},x)\|_2&=\left\|\phi(s_0,a_0)\sum_{i=0}^{n-1}\gamma^i\prod_{j=1}^ic(s_j,a_j)\mathcal{R}(s_i,a_i)\right\|_2\leq \sum_{i=0}^{n-1}\gamma^ic_{\max}^i\leq f_n(\gamma c_{\max}).
	\end{align*}
	\item Under Assumption \ref{as:MC}, it is clear that the stationary distribution $\nu$ of the Markov chain $\{X_k\}$ is given by
	\begin{align*}
		\nu(s_0,a_0,...,s_n,a_n)=\mu(s_0)\left(\prod_{i=0}^{n-1}\pi_b(a_i|s_i)P_{a_i}(s_i,s_{i+1})\right)\pi_b(a_n|s_n)
	\end{align*}
for all $(s_0,a_0,...,s_n,a_n)\in\mathcal{X}$.
	Moreover, for any $x=(s_0,a_0,...,s_n,a_n)\in\mathcal{X}$, we have for any $k\geq 0$ that
	\begin{align*}
		\left\|P_X^{k+n+1}(x,\cdot)-\nu(\cdot)\right\|_{\text{TV}}
		=\;&\frac{1}{2}\sum_{s_0',a_0',\cdots,s_n',a_n'}\left|\sum_{s}P_{a_n}(s_n,s)P^k_{\pi_b}(s,s_0')-\mu(s_0')\right|\\
		&\times \left[\prod_{i=0}^{n-1}\pi_b(a_i'\mid s_i')P_{a_i'}(s_i',s_{i+1}')\right]\pi_b(a_n'\mid s_n')\\
		=\;&\frac{1}{2}\sum_{s_0'}\left|\sum_{s}P_{a_n}(s_n,s)P^k_{\pi_b}(s,s_0')-\mu(s_0')\right|\\
		\leq \;&\frac{1}{2}\sum_{s}P_{a_n}(s_n,s)\sum_{s_0'}\left|P^k_{\pi_b}(s,s_0')-\mu(s_0')\right|\\
		\leq \;&\max_{s\in\mathcal{S}}\|P^k_{\pi_b}(s,\cdot)-\mu(\cdot)\|_{\text{TV}}\\
		\leq \;&C\sigma^k.
	\end{align*}
	Since the RHS of the previous inequality does not depend on $x$, we in fact have 
	\begin{align*}
		\max_{x\in\mathcal{X}}\left\|P^{k+n+1}_X(x,\cdot)-\nu(\cdot)\right\|_{\text{TV}}\leq C\sigma^k,\quad \forall\;k\geq 0.
	\end{align*}
	\item Using the fact that $\mathcal{B}_{c,\rho}(\cdot)$ is a linear operator, we have for any $w\in\mathbb{R}^d$ that
	\begin{align*}
		&(w-w_{c,\rho}^\pi)^\top\Bar{F}(w)\\
		=\;& (w-w_{c,\rho}^\pi)^\top\Phi^\top \mathcal{K}_{SA}\left(\mathcal{B}_{c,\rho}(\Phi w)-\Phi w\right)\\
		=\;& (w-w_{c,\rho}^\pi)^\top\Phi^\top \mathcal{K}_{SA}\left(\mathcal{B}_{c,\rho}(\Phi w)-\mathcal{B}_{c,\rho}(\Phi w_{c,\rho}^\pi)\right)-(w-w_{c,\rho}^\pi)^\top\Phi^\top \mathcal{K}_{SA}\Phi (w-w_{c,\rho}^\pi)\\
		=\;& (w-w_{c,\rho}^\pi)^\top\Phi^\top \mathcal{K}_{SA} \Phi (\Phi^\top \mathcal{K}_{SA}\Phi)^{-1}\Phi^\top \mathcal{K}_{SA}\mathcal{B}_{c,\rho}(\Phi (w- w_{c,\rho}^\pi))\\
		&-(w-w_{c,\rho}^\pi)^\top\Phi^\top \mathcal{K}_{SA}\Phi (w-w_{c,\rho}^\pi)\\
		=\;& (w-w_{c,\rho}^\pi)^\top\Phi^\top \mathcal{K}_{SA} \Phi (\Phi^\top \mathcal{K}_{SA}\Phi)^{-1}\Phi^\top \mathcal{K}_{SA}\mathcal{B}_{c,\rho}(\Phi (w- w_{c,\rho}^\pi))\\
		&-(w-w_{c,\rho}^\pi)^\top\Phi^\top \mathcal{K}_{SA}\Phi (w-w_{c,\rho}^\pi)\\
		\leq \;& \|\Phi(w-w_{c,\rho}^\pi)\|_{\mathcal{K}_{SA}} \|\Phi (\Phi^\top \mathcal{K}_{SA}\Phi)^{-1}\Phi^\top \mathcal{K}_{SA}\mathcal{B}_{c,\rho}(\Phi (w- w_{c,\rho}^\pi))\|_{\mathcal{K}_{SA}}\\
		&- \|\Phi (w-w_{c,\rho}^\pi)\|_{\mathcal{K}_{SA}}^2\\
		=\;&\|\Phi(w-w_{c,\rho}^\pi)\|_{\mathcal{K}_{SA}} \|\text{Proj}_{\mathcal{Q}}\mathcal{B}_{c,\rho}(\Phi (w- w_{c,\rho}^\pi))\|_{\mathcal{K}_{SA}}- \|\Phi (w-w_{c,\rho}^\pi)\|_{\mathcal{K}_{SA}}^2\\
		\leq \;&\gamma_c\|\Phi(w-w_{c,\rho}^\pi)\|_{\mathcal{K}_{SA}} \|\Phi (w- w_{c,\rho}^\pi)\|_{\mathcal{K}_{SA}}- \|\Phi (w-w_{c,\rho}^\pi)\|_{\mathcal{K}_{SA}}^2\\
		=\;&-(1-\gamma_c)\|\Phi (w-w_{c,\rho}^\pi)\|_{\mathcal{K}_{SA}}^2\\
		\leq \;&-(1-\gamma_c)\lambda_{\min}\|w-w_{c,\rho}^\pi\|_2^2.
	\end{align*}
\end{enumerate}
\end{proof}

Proposition \ref{prop:PE:properties} (1) establishes the Lipschitz continuity of the operator $F(\cdot,\cdot)$, Proposition \ref{prop:PE:properties} (2) establishes the geometric mixing of the auxiliary Markov chain $\{X_k\}$, and Proposition \ref{prop:PE:properties} (3) essentially guarantees that the ODE $\dot{x}(t)=\Bar{F}(x(t))$ associated with stochastic approximation algorithm (\ref{eq:sa}) is globally geometrically stable.  The rest of the proof follows by applying Theorem 2.1 of \cite{chen2019finitesample} to Algorithm \ref{algorithm:Off-Policy-TD}. In particular, when using constant stepsize (i.e., $\alpha_k\equiv \alpha$) with $\alpha$ chosen such that $\alpha(t_\alpha+n+1)\leq \frac{(1-\gamma_c)\lambda_{\min}}{130L^2}$, we have for all $k\geq t_\alpha+n+1$ that
\begin{align*}
	\mathbb{E}[\|w_k-w_{c,\rho}^\pi\|_2^2]\leq c_1(1-(1-\gamma_c)\lambda_{\min}\alpha)^{k-(t_\alpha+n+1)}+c_2\frac{\alpha L^2(t_\alpha+n+1)}{(1-\gamma_c)\lambda_{\min}},
\end{align*}
where $c_1=(\|w_0\|_2+\|w_0-w_{c,\rho}^\pi\|_2+1)^2$ and $c_2=130(\|w_{c,\rho}^\pi\|_2+1)^2$.  

When using diminishing stepsizes of the form $\alpha_k=\frac{\alpha}{k+h}$ with $\alpha>\frac{1}{(1-\gamma_c)\lambda_{\min}}$ and $h$ chosen such that $\sum_{i=k-(t_k+n+1)}^{k-1}\alpha_i\leq \frac{(1-\gamma_c)\lambda_{\min}}{130L^2}$ for all $k\geq t_k+n+1$, we have 
\begin{align*}
	\mathbb{E}[\|w_k-w_{c,\rho}^\pi\|_2^2]\leq c_1\frac{k_0+h}{k+h}+c_2\frac{8e\alpha^2}{(1-\gamma_c)\lambda_{\min}\alpha-1}\frac{t_k+n+1}{k+h}
\end{align*}
for all $k\geq k_0$, where $k_0=\min\{k\;:\;k\geq t_k+n+1\}$.

The finite-sample bounds of $\lambda$-averaged $Q$-trace and two-sided $Q$-trace directly follow from Theorem \ref{thm:Off-Policy-TD}. To show the performance guarantee (i.e. Eqs. (\ref{eq:Qtrace_performance}) and (\ref{eq:twosidedQ-performance})) on the limit point $w_{c,\rho}^\pi$, we apply Lemma \ref{le:contraction}. Note that when $c(s,a)=\rho(s,a)=\lambda(s)\frac{\pi(a|s)}{\pi_b(a|s)}+1-\lambda(s)$ for all $(s,a)$, we have for any $s\in\mathcal{S}$ that
\begin{align*}
	\sum_{a\in\mathcal{A}}|\pi(a|s)-\pi_b(a|s)\rho(s,a)|=\;& (1-\lambda(s))\sum_{a\in\mathcal{A}}|\pi(a|s)-\pi_b(a|s)|\\
	=\;&(1-\lambda(s))\|\pi(\cdot|s)-\pi_b(\cdot|s)\|_1.
\end{align*}
When $c(s,a)=\rho(s,a)=g_{\ell(s),u(s)}(\pi(a|s)/\pi_b(a|s))$ for all $(s,a)$, we have
\begin{align*}
	&\sum_{a\in\mathcal{A}}|\pi(a|s)-\pi_b(a|s)\rho(s,a)|\\
	\leq \;&\sum_{a\in\mathcal{A}}|(\pi(a|s)-\pi_b(a|s)\ell(s))\mathds{1}\{\pi(a|s)<\ell(s)\pi_b(a|s)\}|\\
	&+\sum_{a\in\mathcal{A}}|(\pi(a|s)-\pi_b(a|s)u(s))\mathds{1}\{\pi(a|s)>u(s)\pi_b(a|s)\}|\\
	=\;&\sum_{a\in\mathcal{A}}\max(\pi(a|s)-\pi_b(a|s)u(s),0)-\min(\pi(a|s)-\pi_b(a|s)\ell(s),0)\\
	=\;&\sum_{a\in\mathcal{A}}(u_{\pi,\pi_b}(s,a)-\ell_{\pi,\pi_b}(s,a)).
\end{align*}
This completes the proof.

\subsection{Proof of Theorem \ref{thm:combine}}\label{subsec:pf:main}

In view of Theorem  \ref{thm:main}, it remains to control the term $N_2$ using Theorem \ref{thm:Off-Policy-TD}. Using triangle inequality and we have for any $0\leq t\leq T-1$:
\begin{align}
	&\mathbb{E}[\|Q^{\pi_t}-\Phi w_t\|_\infty]\\
	\leq\;& \mathbb{E}[\|Q^{\pi_t}-\Phi w_{c,\rho}^{\pi_t}\|_\infty]+\mathbb{E}[\|\Phi (w_{c,\rho}^{\pi_t}- w_t)\|_\infty]\nonumber\\
	\leq\;& \mathbb{E}[\|Q^{\pi_t}-\Phi w_{c,\rho}^{\pi_t}\|_\infty]+\|\Phi\|_\infty\mathbb{E}[ \|w_{c,\rho}^{\pi_t}- w_t\|_\infty]\nonumber\\
	\leq\;& \mathbb{E}[\|Q^{\pi_t}-\Phi w_{c,\rho}^{\pi_t}\|_\infty]+\mathbb{E}[ \|w_{c,\rho}^{\pi_t}- w_t\|_\infty]\tag{$\|\Phi\|_\infty\leq 1$}\\
	\leq\;& \mathbb{E}[\|Q_{c,\rho}^{\pi_t}-\Phi w_{c,\rho}^{\pi_t}\|_\infty]+\mathbb{E}[\|Q_{c,\rho}^{\pi_t}-Q^{\pi_t}\|_\infty]+\mathbb{E}[ \|w_{c,\rho}^{\pi_t}- w_t\|_\infty]\nonumber\\
	\leq\;& \mathcal{E}_{\text{approx}}+\frac{\gamma}{(1-\gamma)^2}\max_{s\in\mathcal{S}}(1-\lambda(s))\|\pi_t(\cdot|s)-\pi_b(\cdot|s)\|_1+\mathbb{E}[ \|w_{c,\rho}^{\pi_t}- w_t\|_\infty]\tag{Apply Eq. (\ref{eq:169})}\\
	\leq\;& \mathcal{E}_{\text{approx}}+\frac{\gamma}{(1-\gamma)^2}\mathcal{E}_{\text{bias}}+\mathbb{E}[ \|w_{c,\rho}^{\pi_t}- w_t\|_\infty]\label{eq:90}.
\end{align}
To control $\mathbb{E}[ \|w_{c,\rho}^{\pi_t}- w_t\|_\infty]$, using Theorem \ref{co:Qtrace-bound} and we have for all $t=0,\cdots,T-1$ that 
\begin{align*}
	\mathbb{E}[\|w_{c,\rho}^{\pi_t}-w_t\|_\infty]\leq \;&\mathbb{E}[\|w_{c,\rho}^{\pi_t}-w_t\|_2]\\
	\leq \;&(\mathbb{E}[\|w_{c,\rho}^{\pi_t}-w_t\|_2^2])^{1/2}\tag{Jensen's Inequality}\\
	\leq \;& c_{1,i}(1-(1-\gamma_c)\lambda_{\min}\alpha)^{\frac{1}{2}[K-(t_\alpha+n+1)]}+c_{2,i}\frac{[\alpha(t_\alpha+n+1)]^{1/2}}{\sqrt{1-\gamma_c}\sqrt{\lambda_{\min}}},
\end{align*}
where 
$c_{1,t}=\|w^{\pi_t}_{c,\rho}\|_2+1$ and $c_{2,t}=11.5L(\|w_{c,\rho}^{\pi_t}\|_2+1)$. Note that the last line of the previous inequality follows from $\sqrt{a+b}\leq \sqrt{a}+\sqrt{b}$ for any $a,b\geq 0$. To further control the constants $c_{1,t}$ and $c_{2,t}$, note that we have for any policy $\pi$ that
\begin{align*}
	\|w_{c,\rho}^\pi\|_2&\leq \frac{1}{\sqrt{\lambda_{\min}}}\|\Phi w_{c,\rho}^\pi\|_{\mathcal{K}_{SA}}\\
	&\leq \frac{1}{\sqrt{\lambda_{\min}}}\left(\|Q_{c,\rho}^\pi\|_{\mathcal{K}_{SA}}+\frac{1}{\sqrt{1-\gamma_c^2}(1-\gamma )}\right)\tag{Eq. (\ref{eq:5})}\\
	&\leq \frac{1}{\sqrt{\lambda_{\min}}}\left(\frac{1}{1-\gamma }+\frac{1}{\sqrt{1-\gamma_c^2}(1-\gamma )}\right)\\
	&\leq \frac{2}{\sqrt{\lambda_{\min}}(1-\gamma )\sqrt{1-\gamma_c}}.
\end{align*}
Therefore we have $c_{1,t}\leq \frac{3}{\sqrt{\lambda_{\min}}(1-\gamma )\sqrt{1-\gamma_c}}$ and $c_{2,t}\leq \frac{35L}{\sqrt{\lambda_{\min}}(1-\gamma )\sqrt{1-\gamma_c}}$ for any $0\leq t\leq T-1$. Substituting the upper bound we obtained for $\mathbb{E}[\|w_{c,\rho}^{\pi_t}-w_t\|_\infty]$ into Eq. (\ref{eq:90}) and we have for any $0\leq i\leq T-1$:
\begin{align*}
	\mathbb{E}[\|Q^{\pi_t}-\Phi w_t\|_\infty]\leq\;& \mathcal{E}_{\text{approx}}+\frac{\gamma\mathcal{E}_{\text{bias}}}{(1-\gamma)^2}+\frac{3(1-(1-\gamma_c)\lambda_{\min}\alpha)^{\frac{1}{2}[K-(t_\alpha+n+1)]}}{\sqrt{\lambda_{\min}}(1-\gamma )\sqrt{1-\gamma_c}}\\
	&+\frac{35L[\alpha(t_\alpha+n+1)]^{1/2}}{(1-\gamma)(1-\gamma_c)\lambda_{\min}}.
\end{align*}
Finally, using the previous inequality and we obtain the following bound on the term $N_2$ in Theorem \ref{thm:main}:
\begin{align*}
	N_2=\;&\frac{2\gamma}{1-\gamma}\sum_{t=0}^{T-1}\gamma^{T-1-t}\mathbb{E}[\|Q^{\pi_t}-\Phi w_t\|_\infty]\\
	\leq \;&\frac{2\gamma\mathcal{E}_{\text{approx}}}{(1-\gamma)^2} +\frac{2\gamma^2\mathcal{E}_{\text{bias}}}{(1-\gamma)^4}+\frac{6(1-(1-\gamma_c)\lambda_{\min}\alpha)^{\frac{1}{2}[K-(t_\alpha+n+1)]}}{(1-\gamma)^3(1-\gamma_c)^{1/2}\lambda_{\min}^{1/2}}\\
	&+\frac{70L[\alpha(t_\alpha+n+1)]^{1/2}}{\lambda_{\min}(1-\gamma_c)(1-\gamma)^3}.
\end{align*}
The rest of the proof follows by using the upper bound we obtained for the term $N_2$ in Theorem \ref{thm:main}.

\section{Conclusion}\label{sec:conclusion}

In this work, we study finite-sample guarantees of general policy-based algorithms (especially natural policy gradient) under off-policy sampling and linear function approximation. By viewing natural policy gradient as an approximate version of PI, we establish its geometric convergence without requiring any regularization. As for the critic, to overcome the deadly triad and the high variance in off-policy learning, we design a convergent framework of single time-scale TD-learning algorithms, including two specific algorithms called $\lambda$-averaged $Q$-trace and two-sided $Q$-trace. After combining the actor and the critic, we obtain an overall sample complexity bound of $\Tilde{\mathcal{O}}(\epsilon^{-2})$.

\bibliographystyle{apalike}
{\small \bibliography{references}}

\begin{center}
	\large \textbf{Appendices}
\end{center}

\appendix

\section{Proof of All Technical Lemmas}

\subsection{Proof of Lemma \ref{le:difference}}\label{pf:le:difference}
For any $\beta>0$, consider the function $h_\beta:\mathbb{R}^d\mapsto\mathbb{R}$ defined by
\begin{align*}
	h_\beta(x)=\frac{1}{\beta}\log\left(\sum_{i=1}^dy_ie^{\beta x_i}\right).
\end{align*}
Assume without loss of generality that $i_{\max}=1$. Then it is clear that $h_\beta(x)\leq x_1$. On the other hand, we have
\begin{align}\label{eq:167}
	x_1\leq  \frac{1}{\beta}\log\left(\sum_{i=1}^d\frac{y_i}{y_1} e^{\beta x_i}\right)=h_\beta(x)+\frac{1}{\beta}\log\left(\frac{1}{y_1}\right).
\end{align}
Since it is well-known that $h_\beta(x)$ is a convex differentiable function, we have for any $x\in\mathbb{R}^d$ that $h_\beta(\bm{0})-h_\beta(x)\geq \langle\nabla h_\beta(x),-x\rangle$, which implies
\begin{align}\label{eq:166}
	\langle\nabla h_\beta(x),x\rangle=\frac{\sum_{i=1}^dx_iy_ie^{\beta x_i}}{\sum_{j=1}^dy_je^{\beta x_j}}\geq h_\beta(x)-h_\beta(0)=h_\beta(x).
\end{align}
Using Eqs. (\ref{eq:167}) and (\ref{eq:166}) and we finally obtain
\begin{align*}
	\max_{1\leq i\leq d}x_i-\frac{\sum_{i=1}^dx_iy_ie^{\beta x_i}}{\sum_{j=1}^dy_je^{\beta x_j}}
	\leq x_1-h_\beta (x)\leq \frac{1}{\beta}\log\left(\frac{1}{y_1}\right).
\end{align*}

\subsection{Proof of Lemma \ref{le:equation}}
We begin by introducing some notation. Let $\pi_c$ and $\pi_\rho$ be two policies defined by
\begin{align*}
	\pi_c(a|s)=\frac{\pi_b(a|s)c(s,a)}{\sum_{a'\in\mathcal{A}}\pi_b(a'|s)c(s,a')},\quad\text{ and }\quad\pi_\rho(a|s)=\frac{\pi_b(a|s)\rho(s,a)}{\sum_{a'\in\mathcal{A}}\pi_b(a'|s)\rho(s,a')},\quad \forall\;(s,a).
\end{align*}
Let $P_{\pi_c}$ and $P_{\pi_\rho}$ be the transition probability matrices of the Markov chain $\{S_k\}$ induced by the policies $\pi_c$ and $\pi_{\rho}$, respectively. Then, Eq. (\ref{eq:11}) can be compactly written in vector form as
\begin{align*}
	\Phi^\top \mathcal{K}_{SA}\sum_{i=0}^{n-1}(\gamma P_{\pi_c}D_c)^i(R+\gamma P_{\pi_\rho}D_\rho \Phi w-\Phi w)=0,
\end{align*}
where $R\in\mathbb{R}^{|\mathcal{S}||\mathcal{A}|}$ is defined by $R(s,a)=\mathcal{R}(s,a)$ for all $(s,a)$. Observe that the above equation is further equivalent to
\begin{align}\label{eq:12}
	\Phi(\Phi^\top \mathcal{K}_{SA}\Phi)^{-1}\Phi^\top \mathcal{K}_{SA}\sum_{i=0}^{n-1}(\gamma P_{\pi_c}D_c)^i(R+\gamma P_{\pi_\rho}D_\rho \Phi w-\Phi w)=0.
\end{align}
To see this, note that the matrix $\Phi$ has full column-rank, and the matrix $\Phi^\top \mathcal{K}_{SA}\Phi$ is positive definite and hence invertible. Therefore, we have $x=\bm{0}$ if and only if $\Phi(\Phi^\top \mathcal{K}_{SA}\Phi)^{-1} x=\bm{0}$. To rewrite Eq. (\ref{eq:12}) in the desired form of the generalized PBE (\ref{eq:pbe-key}), we use the following three observations. 
\begin{enumerate}[(1)]
	\item The projection operator $\text{Proj}_{\mathcal{Q}}(\cdot)$ is explicitly given by $\text{Proj}_{\mathcal{Q}}(\cdot)=\Phi(\Phi^\top \mathcal{K}_{SA}\Phi)^{-1}\Phi^\top\mathcal{K}_{SA}(\cdot)$.
	\item The operator $\mathcal{T}_c(\cdot)$ is explicitly given by $\mathcal{T}_c(\cdot)=\sum_{i=0}^{n-1}(\gamma P_{\pi_c}D_c)^i(\cdot)$.
	\item The operator $\mathcal{H}_\rho(\cdot)$ is explicitly given by $\mathcal{H}_\rho(\cdot)=R+\gamma P_{\pi_\rho}D_\rho(\cdot)$.
\end{enumerate}
Therefore, Eq. (\ref{eq:12}) is equivalent to
\begin{align}\label{eq:13}
	\text{Proj}_{\mathcal{Q}} [\mathcal{T}_c(\mathcal{H}_\rho(\Phi w)-\Phi w)]=0.
\end{align}
Finally, adding and subtracting $\Phi w$ on both sides of the previous inequality and we obtain the desired generalized PBE:
\begin{align*}
	\Phi w&= \text{Proj}_{\mathcal{Q}} [\mathcal{T}_c(\mathcal{H}_\rho(\Phi w)-\Phi w)]+\Phi w\\
	&=\text{Proj}_{\mathcal{Q}} [\mathcal{T}_c(\mathcal{H}_\rho(\Phi w)-\Phi w)+\Phi w]\\
	&=\text{Proj}_{\mathcal{Q}}\mathcal{B}_{c,\rho}(\Phi w),
\end{align*}
where the second equality follows from (1) $\Phi w\in\mathcal{Q}$ and (2) $\text{Proj}_{\mathcal{Q}}(\cdot)$ is a linear operator.

\subsection{Proof of Lemma \ref{le:Lipschitz_factor}}
For any $Q_1,Q_2\in\mathbb{R}^{|\mathcal{S}||\mathcal{A}|}$, using the fact that $\text{Proj}_{\mathcal{Q}}$ is non-expansive with respect to $\|\cdot\|_{\mathcal{K}_{SA}}$, we have
\begin{align*}
	\|\text{Proj}_{\mathcal{Q}}\mathcal{B}_{c,\rho}(Q_1)-\text{Proj}_{\mathcal{Q}}\mathcal{B}_{c,\rho}(Q_2)\|_{\mathcal{K}_{SA}}
	\leq \;&\|\mathcal{B}_{c,\rho}(Q_1)-\mathcal{B}_{c,\rho}(Q_2)\|_{\mathcal{K}_{SA}}\\
	\leq \;&\|\mathcal{B}_{c,\rho}(Q_1)-\mathcal{B}_{c,\rho}(Q_2)\|_\infty\tag{$\|\cdot\|_{\mathcal{K}_{SA}}\leq \|\cdot\|_\infty$}\\
	\leq \;&\tilde{\gamma}(n)\|Q_1-Q_2\|_\infty\\
	\leq \;&\frac{\tilde{\gamma}(n)}{\sqrt{\mathcal{K}_{SA,\min}}}\|Q_1-Q_2\|_{\mathcal{K}_{SA}}\tag{$\|\cdot\|_\infty\leq \frac{1}{\sqrt{\mathcal{K}_{SA,\min}}}\|\cdot\|_{\mathcal{K}_{SA}}$},
\end{align*}
where the third inequality follows from $\mathcal{B}_{c,\rho}(\cdot)$ being a $\tilde{\gamma}(n)$-contraction operator with respect to $\|\cdot\|_\infty$ \citep{chen2021off} \footnote{\cite{chen2021off} works with an asynchronous variant of the generalized Bellman operator, which is shown to be a contraction mapping with respect to $\|\cdot\|_\infty$ with contraction factor $1-\mathcal{K}_{SA,\min}f_n(\gamma D_{c,\min})(1-\gamma D_{\rho,\max})$. In this paper we work with the synchronous generalized Bellman operator $\mathcal{B}_{c,\rho}(\cdot)$. In this case, one can easily verify that the corresponding contraction factor can be obtained by simply dropping the factor  $\mathcal{K}_{SA,\min}$.}.

\subsection{Proof of Lemma \ref{le:contraction}}
We first show that under Condition \ref{as:IS_ratio} (3), we have $\lim_{n\rightarrow\infty}\tilde{\gamma}(n)/\sqrt{\mathcal{K}_{SA,\min}}<1$. Using the explicit expression of  $\tilde{\gamma}(n)$, we have
\begin{align*}
	\lim_{n\rightarrow\infty}\frac{\tilde{\gamma}(n)}{\sqrt{\mathcal{K}_{SA,\min}}}
	&=\lim_{n\rightarrow\infty}\frac{1-f_n(\gamma D_{c,\min})(1-\gamma D_{\rho,\max})}{\sqrt{\mathcal{K}_{SA,\min}}}\\
	&=\lim_{n\rightarrow\infty}\frac{1-\frac{1-(\gamma D_{c,\min})^n}{1-\gamma D_{c,\min}}(1-\gamma D_{\rho,\max})}{\sqrt{\mathcal{K}_{SA,\min}}}\tag{$f_n(x)=\sum_{i=0}^{n-1}x^i$ and $\gamma D_{c,\min}<1$}\\
	&=\frac{\gamma (D_{\rho,\max}-D_{c,\min})}{(1-\gamma D_{c,\min})\sqrt{\mathcal{K}_{SA,\min}}}\\
	&<1. \tag{Condition \ref{as:IS_ratio} (3)}
\end{align*}
Therefore, when $n$ is chosen such that $\gamma_c=\frac{\tilde{\gamma}(n)}{\sqrt{\mathcal{K}_{SA,\min}}}<1$, we have by Lemma \ref{le:Lipschitz_factor} that
\begin{align*}
	\|\text{Proj}_{\mathcal{Q}}\mathcal{B}_{c,\rho}(Q_1)-\text{Proj}_{\mathcal{Q}}\leq \gamma_c\|Q_1-Q_2\|_{\mathcal{K}_{SA}},\quad \forall\;Q_1,Q_2\in\mathbb{R}^{|\mathcal{S}||\mathcal{A}|}.
\end{align*}
It follows that the composed operator $\text{Proj}_{\mathcal{Q}}\mathcal{B}_{c,\rho}(\cdot)$ is a contraction mapping with respect to $\|\cdot\|_{\mathcal{K}_{SA}}$, with contraction factor $\gamma_c$. 

Next consider the difference between $Q^\pi$ and $\Phi w_{c,\rho}^\pi$. First of all, we have by triangle inequality that
\begin{align}
	\|Q^\pi-\Phi w_{c,\rho}^\pi\|_{\mathcal{K}_{SA}}&=\|Q^\pi-Q_{c,\rho}^\pi+Q_{c,\rho}^\pi-\Phi w_{c,\rho}^\pi\|_{\mathcal{K}_{SA}}\nonumber\\
	&\leq \|Q^\pi-Q_{c,\rho}^\pi\|_{\mathcal{K}_{SA}}+\|Q_{c,\rho}^\pi-\Phi w_{c,\rho}^\pi\|_{\mathcal{K}_{SA}}.\label{eq:70}
\end{align}
We next bound each term on the RHS of the previous inequality. For the first term, it was already established in Proposition 2.1 of \cite{chen2021off} that
\begin{align}\label{eq:169}
	\|Q^\pi-Q_{c,\rho}^\pi\|_{\mathcal{K}_{SA}
	}\leq \|Q^\pi-Q_{c,\rho}^\pi\|_\infty&\leq \frac{\gamma\max_{s\in\mathcal{S}}\sum_{a\in\mathcal{A}}|\pi(a|s)-\pi_b(a|s)\rho(s,a)|}{(1-\gamma)(1-\gamma D_{\rho,\max})}.
\end{align}
Now consider the second term on the RHS of Eq. (\ref{eq:70}). First note that
\begin{align*}
	&\|Q_{c,\rho}^\pi-\Phi w_{c,\rho}^\pi\|_{\mathcal{K}_{SA}}^2\\
	=\;&\|Q_{c,\rho}^\pi-\text{Proj}_{\mathcal{Q}}Q_{c,\rho}^\pi+\text{Proj}_{\mathcal{Q}}Q_{c,\rho}^\pi-\Phi w_{c,\rho}^\pi\|_{\mathcal{K}_{SA}}^2\\
	=\;&\|Q_{c,\rho}^\pi-\text{Proj}_{\mathcal{Q}}Q_{c,\rho}^\pi\|_{\mathcal{K}_{SA}}^2+\|\text{Proj}_{\mathcal{Q}}Q_{c,\rho}^\pi-\Phi w_{c,\rho}^\pi\|_{\mathcal{K}_{SA}}^2\tag{$*$}\\
	=\;&\|Q_{c,\rho}^\pi-\text{Proj}_{\mathcal{Q}}Q_{c,\rho}^\pi\|_{\mathcal{K}_{SA}}^2+\|\text{Proj}_{\mathcal{Q}}\mathcal{B}_{c,\rho}(Q_{c,\rho}^\pi)-\text{Proj}_{\mathcal{Q}}\mathcal{B}_{c,\rho}(\Phi w_{c,\rho}^\pi)\|_{\mathcal{K}_{SA}}^2\\
	\leq\;& \|Q_{c,\rho}^\pi-\text{Proj}_{\mathcal{Q}}Q_{c,\rho}^\pi\|_{\mathcal{K}_{SA}}^2+\gamma_c^2\|Q_{c,\rho}^\pi-\Phi w_{c,\rho}^\pi\|_{\mathcal{K}_{SA}}^2,
\end{align*}
where Eq. ($*$) follows from the Babylonian–Pythagorean theorem (observe that $Q_{c,\rho}^\pi-\text{Proj}_{\mathcal{Q}}Q_{c,\rho}^\pi\perp \mathcal{Q}$ and $\text{Proj}_{\mathcal{Q}}Q_{c,\rho}^\pi-\Phi w_{c,\rho}^\pi \in\mathcal{Q}$). Rearrange the previous inequality and we have
\begin{align}\label{eq:5}
	\|Q_{c,\rho}^\pi-\Phi w_{c,\rho}^\pi\|_{\mathcal{K}_{SA}}&\leq \frac{1}{\sqrt{1-\gamma_c^2}}\|Q_{c,\rho}^\pi-\text{Proj}_{\mathcal{Q}}Q_{c,\rho}^\pi\|_{\mathcal{K}_{SA}}.
\end{align}
Substituting Eqs. (\ref{eq:169}) and (\ref{eq:5}) into the RHS of Eq. (\ref{eq:70}) and we finally obtain
\begin{align*}
	\|Q^\pi-\Phi w_{c,\rho}^\pi\|_{\mathcal{K}_{SA}}
	\leq\;& \frac{\gamma\max_{s\in\mathcal{S}}\sum_{a\in\mathcal{A}}|\pi(a|s)-\pi_b(a|s)\rho(s,a)|}{(1-\gamma)(1-\gamma D_{\rho,\max})}\\
	&+\frac{1}{\sqrt{1-\gamma_c^2}}\|Q_{c,\rho}^\pi-\text{Proj}_{\mathcal{Q}}Q_{c,\rho}^\pi\|_{\mathcal{K}_{SA}}.
\end{align*}

\section{The High Variance in \cite{chen2021NACLFA}}\label{ap:compare}

Consider Theorem 2.1 of \cite{chen2021NACLFA}. The constant $c_2$ on the second term is proportional to $\sum_{i=0}^{n-1}(\gamma \max_{s,a}\frac{\pi(a|s)}{\pi_b(a|s)})^i$ (which appears as $f(\gamma \zeta_\pi)$ using the notation of \cite{chen2021NACLFA}). When $\frac{\pi(a|s)}{\pi_b(a|s)}>1/\gamma$ (which can usually happen in practice), the parameter $c_2$ grows exponentially fast with respect to the bootstrapping parameter $n$. Moreover, since $n$ needs to be chosen large enough for the results in \cite{chen2021NACLFA} to hold, the variance term on the finite-sample bound of the $n$-step off-policy TD-learning algorithm with linear function approximation is exponentially large.

\end{document}